%% file: NonparametricLearning.tex
\documentclass{colt2017} 

\newif\ifarxiv
\arxivtrue 

\title[Algorithmic Chaining and the Role of Partial Feedback in Online Nonparametric Learning]{Algorithmic Chaining and the Role of Partial Feedback \\ in Online Nonparametric Learning}

\usepackage[utf8]{inputenc} 
\usepackage[T1]{fontenc}    
\usepackage{times}
\usepackage{enumitem}
\usepackage{fancybox}
\usepackage{tikz}
\usetikzlibrary{decorations.pathreplacing}
\usepackage{easy-todo}
\usepackage{pgfplots}
\usepackage{booktabs} 
\usepackage{multirow}
\pgfplotsset{width=7cm,compat=1.8}
\usepackage{cancel}

 \coltauthor{\Name{Nicol\`o Cesa-Bianchi} \Email{nicolo.cesa-bianchi@unimi.it} \\
	\addr Universit{\`a} degli Studi di Milano, Milano, Italy
	\AND
	\Name{Pierre Gaillard} \Email{pierre.gaillard@inria.fr} \\
	\addr INRIA - Sierra Project-team, Département d'Informatique de l'École Normale Supérieure, Paris, France
	\AND
	\Name{Claudio Gentile} \Email{claudio.gentile@uninsubria.it} \\
	\addr Universit{\`a} degli Studi dell'Insubria, Varese, Italy
 \AND
 \Name{S\'{e}bastien Gerchinovitz} \Email{sebastien.gerchinovitz@math.univ-toulouse.fr}\\
 \addr Universit\'e Toulouse III - Paul Sabatier, Toulouse, France
}

\newenvironment{proofref}[1]{\par\noindent{\bfseries\upshape Proof (of~#1)\ }}{\hfill\BlackBox\\[2mm]}
\newenvironment{sketchproof}[1]{\par\noindent{\bfseries\upshape Sketch of proof (of~#1)\ }}{\hfill\BlackBox\\[2mm]}

\renewcommand{\epsilon}{\varepsilon}
\renewcommand{\leq}{\leqslant}
\renewcommand{\geq}{\geqslant}
\renewcommand{\phi}{\varphi}
\renewcommand{\hat}{\widehat}
\renewcommand{\tilde}{\widetilde}

\newcommand{\R}{\mathbb{R}}

\newcommand{\cX}{\mathcal{X}}
\newcommand{\cF}{\mathcal{F}}

\newcommand{\cK}{\mathcal{K}}
\newcommand{\cN}{\mathcal{N}}
\newcommand{\Prob}{\mathbb{P}}
\newcommand{\E}{\mathbb{E}}
\newcommand{\iid}{{i.i.d.} }
\newcommand{\indicator}[1]{\mathbb{I}_{#1}}
\newcommand{\dd}{\mathrm{d}}

\newcommand{\cC}{\mathcal{C}}

\newcommand{\cL}{\mathcal{L}}
\newcommand{\cO}{\mathcal{O}}
\newcommand{\cP}{\mathcal{P}}
\newcommand{\cT}{\mathcal{T}}
\newcommand{\cY}{\mathcal{Y}}

\newcommand{\eqdef}{\triangleq}
\newcommand{\norm}[1]{\left\lVert#1\right\rVert}
\DeclareMathOperator{\Reg}{Reg}

\DeclareMathOperator*{\argmin}{arg\,min}

\DeclareMathOperator{\automEWA}{\mathrm{Exp4}}
\DeclareMathOperator{\depth}{\mathrm{depth}}
\DeclareMathOperator{\pen}{\mathrm{pen}}

\newcommand{\ve}{\varepsilon}
\newcommand{\scX}{\mathcal{X}}
\newcommand{\yhat}{\widehat{y}}
\newcommand{\ymin}{y^{\mathrm{min}}}
\newcommand{\theset}[2]{ \left\{ {#1} \,:\, {#2} \right\} }
\newcommand{\loss}{\ell}
\newcommand{\xbar}{\bar{x}}
\newcommand{\ybar}{\bar{y}}

\newlength{\minipagewidth}
\newcommand{\blfootnote}[1]{%
  \begingroup
  \renewcommand\thefootnote{}\footnote{\!\!\!\!{#1}}%
  \addtocounter{footnote}{-1}%
  \endgroup
}

\makeatletter
\newcommand{\pushright}[1]{\ifmeasuring@#1\else\omit\hfill$\displaystyle#1$\fi\ignorespaces}
\newcommand{\pushleft}[1]{\ifmeasuring@#1\else\omit$\displaystyle#1$\hfill\fi\ignorespaces}
\makeatother

\begin{document}
\maketitle

\begin{abstract}
%

We \ifarxiv\blfootnote{Accepted for presentation at Conference on
Learning Theory (COLT) 2017.}\fi investigate contextual online learning with nonparametric (Lipschitz) comparison classes under different assumptions on losses and feedback information. For full information feedback and Lipschitz losses, we design the first explicit algorithm achieving the minimax regret rate (up to log factors). In a partial feedback model motivated by second-price auctions, we obtain algorithms for Lipschitz and semi-Lipschitz losses with regret bounds improving on the known bounds for standard bandit feedback. Our analysis combines novel results for contextual second-price auctions with a novel algorithmic approach based on chaining. When the context space is Euclidean, our chaining approach is efficient and delivers an even better regret bound.
\end{abstract}

\begin{keywords}
online learning, nonparametric, chaining, bandits.
\end{keywords}


\section{Introduction}
\input{section-introduction}


\section{Related Work}\label{sec:related}
\input{section-related}


\section{Warmup: Nonparametric Bandits}
\label{sec:bandits}
\input{section-bandits}


\section{One-Sided Full Information Feedback}
\label{sec:sideInformation}
\input{section-sideinfo-introwarmup}
\input{section-sideinfo-chaining}
\input{subsection-efficient-chaining}


\section{A Tight Bound for Full Information through an Explicit Algorithm}
\label{sec:full-info}
\input{section-fullinfo}




\acks{\sloppypar{This work was partially supported by the CIMI (Centre International de Math\'{e}matiques et d'Informatique) Excellence program. Pierre Gaillard and S\'{e}bastien Gerchinovitz acknowledge  the  support  of  the  French  Agence  Nationale  de  la Recherche  (ANR),  under  grants ANR-13-BS01-0005 (project SPADRO) and ANR-13-CORD-0020 (project ALICIA).}}

\bibliography{RefsNonparametricLearning,nicolo}

\ifarxiv
\appendix


\input{appendix-lemmas}

\input{appendix-morelemmas}

\fi

\end{document}

%% file: section-introduction.tex

In online learning \citep{cbl06,shalev2011online,hazan2016introduction} an agent (or learner) interacts with an unknown and arbitrary environment in a sequence of rounds. At each round, the learner chooses an action from a given action space and incurs the loss associated with the chosen action. The loss functions, which are different in each round, are fixed by the environment at the beginning of the interaction. After choosing an action, the learner observes some feedback, which can be used to reduce his loss in subsequent rounds. A variety of different feedback models are discussed in the literature. The most common feedback model is full information, also known as prediction with expert advice, where the learner gets access to the entire loss function at the end of each round. Another common feedback model is bandit information, where the learner just observes the loss assigned to the action chosen in the current round. Feedback models in between full and bandit information are also possible, and can be used to describe many interesting online learning applications ---see e.g., \citep{AlonCGMMS14,alon2015online}. The performance of an online learner is measured using a notion of regret, which is typically defined as the amount by which the learner's cumulative loss exceeds the cumulative loss of the best fixed action in hindsight.

Online contextual learning is a generalization of online learning where the loss functions generated by the environment are paired with contexts from a given context space. On each round, before choosing an action, the learner observes the current context. In the presence of contextual information, the learner's regret is no longer defined against the best action in hindsight, but rather against the best policy (i.e., mapping from the context space to the action space) in a given reference class of policies. In agreement with the online learning framework, online contextual learning is nonstochastic. Namely, regret bounds must hold for arbitrary sequences of contexts and losses.

In order to capture complex environments, the reference class of policies should be as large as possible. In this work, we focus on nonparametric classes of policies, such as classes containing policies that are Lipschitz with respect to metrics defined on the context and action spaces. The best possible (minimax) growth rate of the regret, as a function of the number $T$ of rounds, is then determined by the interplay among the richness of the policy class, the constraints on the loss functions (e.g., Lipschitz, convex, etc.), and the type of feedback information (full, bandit, or in between). Whereas most of the previous works study online nonparametric learning with convex losses, in this paper we investigate nonparametric regret rates for general Lipschitz losses (in fact, some of our results apply to an even larger class of loss functions).

In the full information setting, a very general yet simple algorithmic approach to online nonparametric learning with convex and Lipschitz losses was introduced by \citet{HaMe-07colt-OnlineLearningPriorKnowledge}. For any reference class of Lipschitz policies, they proved a $\tilde{\cO}\big(T^{(d+1)/(d+2)}\big)$ upper bound\footnote{This bound has a polynomial dependence on the metric dimension of the action space, which is absorbed by the asymptotic notation.} on the regret for any context space of metric dimension $d$, where the $\tilde{\cO}$ notation hides logarithmic factors in $T$. In the same work, they also proved a $\Omega\big(T^{(d-1)/d}\big)$ lower bound. The gap between the upper and lower bound was closed by \citet{rakhlin2015online} for arbitrary Lipschitz (not necessarily convex) losses, showing that $T^{(d-1)/d}$ is indeed the minimax rate for full information. Yet, since their approach is nonconstructive, they did not give an explicit algorithm achieving this bound.

As noted elsewhere ---see, e.g., \citep{sliv14}--- the approach of \citet{HaMe-07colt-OnlineLearningPriorKnowledge} can be also adapted to prove a $\tilde{\cO}\big(T^{(d+p+1)/(d+p+2)}\big)$ upper bound on the regret against any class of Lipschitz policies in the bandit information setting with Lipschitz losses, where $p$ is the metric dimension of the action space. The lower bound $\Omega\big(T^{(p+1)/(p+2)}\big)$ proven for $d=0$ \citep{bubeck2011x,ksu08} rules out the possibility of improving the dependence on $p$ in the upper bound.

\paragraph{Our contributions.}
In the full information model, we show the first explicit algorithm achieving the minimax regret rate $\tilde{\cO}\big(T^{(d-1)/d}\big)$ for Lipschitz policies and Lipschitz losses (excluding logarithmic factors in $T$ and polynomial factors in the metric dimension of the action space). When the context space is $[0,1]^d$, our algorithm can be implemented efficiently (i.e., with a running time polynomial in $T$).

Motivated by a problem in online advertising where the action space is the $[0,1]$ interval, we also study a ``one-sided'' full information model in which the loss of each action greater than or equal to the chosen action is available to the learner after each round. For this feedback model, which lies between full and bandit information, we prove a regret bound for Lipschitz policies and Lipschitz losses of order $\tilde{\cO}\big(T^{d/(d+1)}\big)$, which is larger than the minimax regret for full information but smaller than the upper bound for bandit information when $p=1$. For the special case when the context space is $[0,1]^d$, we use a specialized approach offering the double advantage of an improved $\tilde{\cO}\big(T^{(d-1/3)/(d+2/3)}\big)$ regret bound which is also attained by a time-efficient algorithm.

We then study a concrete application for minimizing the seller's regret in contextual second-price auctions with reserve price, a setting where the loss function is not Lipschitz but only semi-Lipschitz. When the feedback after each auction is the seller's revenue together with the highest bid for the current auction, we prove a $\tilde{\cO}\big(T^{(d+1)/(d+2)}\big)$ regret bound against Lipschitz policies (in this setting, a policy maps contexts to reserve prices for the seller). As a by-product, we show the first $\tilde{\cO}\big(\sqrt{T}\big)$ regret bound on the seller's revenue in context-free second-price auctions under the same feedback model as above. Table~\ref{tab:rates} summarizes our results.
\begin{table}[!th]
\small
\begin{center}
\begin{tabular}{llll} \toprule
\bfseries Feedback model & \bfseries Loss functions \hspace*{10pt} & \hspace*{10pt} \bfseries  Upper bound \hspace*{-1.4cm} &  \\ \midrule
     \multirow{2}{*}{Bandit} & Lipschitz & $T^{\frac{d+2}{d+3}}$ & (Theorem~\ref{th:flat-lip}) \\
     & Convex & $T^{\frac{d+1}{d+2}}$ & (Corollary~\ref{cor:convex}) \\[2pt] \hline 
     & & \\[\dimexpr-\normalbaselineskip+2pt]
     \multirow{2}{*}{One-sided full information} & Semi-Lipschitz & $T^{\frac{d+1}{d+2}}$ &  (Theorem~\ref{t:secondpriceflat}) \\
     & Lipschitz & $T^{\frac{d-1/3}{d+2/3}}$ & (Theorem~\ref{thm:HierExp4star}) \\[2pt] \hline
     & & \\[\dimexpr-\normalbaselineskip+2pt]
     Full information & Lipschitz & $T^{\frac{d-1}{d}}$ & (Theorem~\ref{thm:HierarchicalEWA}) \\ 
    \bottomrule
\end{tabular}
\end{center}
\caption{Some regret bounds obtained in this paper. The rates are up to logarithmic factors for Lipschitz policies $f : [0,1]^d \to [0,1]$ with $d \geq 2$. All upper bounds are constructive (i.e., achieved by explict algorithms). The only matching lower bound is the one for full information feedback due to~\cite{HaMe-07colt-OnlineLearningPriorKnowledge}.
}
\label{tab:rates}
\end{table}

In order to prove our results, we approximate the action space using a finite covering (finite coverability is a necessary condition for our results to hold). This allows us to use the many existing algorithms for experts (full information feedback) and bandits when the action space is finite, such as Hedge \citep{FS97} and Exp3/Exp4 \citep{AuCBFrSc-02-NonStochasticBandits}. The simplest of our algorithms, adapted from \citet{HaMe-07colt-OnlineLearningPriorKnowledge}, incrementally covers the context space with balls of fixed radius. Each ball hosts an instance of an online learning algorithm which predicts in all rounds when the context falls into the ball. New balls are adaptively created when new contexts are observed which fall outside the existing balls (see Algorithm~\ref{alg:ContextualExp3} for an example). We use this simple construction to prove the regret bound for contextual second-price auctions, a setting where losses are not Lipschitz. In order to exploit the additional structure provided by Lipschitz losses, we resort to more sophisticated constructions based on chaining \citep{dudley1967sizes}. In particular, inspired by previous works in this area (especially the work of \citealp{GaillardGerchinovitz2015}), we design a chaining-inspired algorithm applied to a hierarchical covering of the policy space. Despite we are not the first ones to use chaining algorithmically in online learning, our idea of constructing a hierarchy of online learners, where each node uses its children as experts, is novel in this context as far as we know. Finally, the time-efficient algorithm achieving the improved regret bound is derived from a different (and more involved) chaining algorithm based on wavelet-like approximation techniques.

\paragraph{Setting and main definitions.}
We assume the context space $\cX$ is a metric space $(\cX,\rho_{\cX})$ of finite metric dimension $d$ and the action space $\cY$ is a metric space $(\cY,\rho_{\cY})$ of finite metric dimension $p$. Hence, there exist $C_{\cX},C_{\cY} > 0$ such that, for all $0 < \ve \leq 1$, $\cX$ and $\cY$ can be covered, respectively, with at most $C_{\cX}\ve^{-d}$ and at most $C_{\cY}\ve^{-p}$ balls of radius $\ve$. For any $0 < \ve \leq 1$, we use $\cY_{\epsilon}$ to denote any $\ve$-covering of $\cY$ of size $K_{\epsilon} \leq C_{\cY}\ve^{-p}$. Finally, we assume that $\cY$ has diameter bounded by $1$ with respect to metric $\rho_{\cY}$.

We consider the following online learning protocol with oblivious adversary and loss functions $\loss_t : \cY \to [0,1]$. Given an unknown sequence $(x_1,\loss_1),(x_2,\loss_2),\dots$ of contexts $x_t\in\cX$ and loss functions $\loss_t : \cY \to [0,1]$, for every round $t = 1,2,\dots$ :
\begin{enumerate}[topsep=0pt,parsep=0pt,itemsep=0pt]
	\item The environment reveals context $x_t \in \cX$;
	\item The learner selects an action $\yhat_t \in \cY$ and incurs loss $\loss_t\big(\yhat_t\big)$;
	\item The learner obtains feedback from the environment.
\end{enumerate}
Loss functions $\loss_t$ satisfy the $1$-Lipschitz\footnote{\label{foot:lip}Assuming a unit Lipschitz constant is without loss of generality because our algorithms are oblivious to it.
}
condition $\big|\loss_t(y) - \loss_t(y')\big| \leq \rho_{\cY}(y,y')$ for all $y,y' \in \cY$. However, we occasionally consider losses satisfying a weaker semi-Lipschitz condition. 

We study three different types of feedback: bandit feedback (the learner only observes the loss $\ell_t(\hat{y}_t)$ of the selected action $\hat{y}_t$), full information feedback (the learner can compute $\ell_t(y)$ for any $y\in\cY$), and one-sided full information feedback ($\cY \equiv [0,1]$, and the learner can compute $\loss_t(y)$ if and only if $y \geq \hat{y}_t$). Given a reference class $\cF \subseteq \cY^{\cX}$ of policies, the learner's goal is to minimize the regret against the best policy in the class,
\[
	\Reg_T(\cF)
\eqdef
	\E\left[\sum_{t=1}^T \ell_t\bigl(\hat{y}_t\bigr)\right] - \inf_{f \in \cF} \sum_{t=1}^T \ell_t\bigl(f(x_t)\bigr)\,,
\]
where the expectation is with respect to the learner's internal randomization.
We derive regret bounds for the competitor class $\cF$ made up of all bounded functions $f : \cX \to \cY$ that are $1$-Lipschitz\footnote{Assuming a unit Lipschitz constant $L_f$ for the functions in $f \in\cF$ is also without loss of generality because our algorithms are oblivious to it. The only exception is algorithm $\textrm{HierExp4}^\star$ of Section~\ref{sec:efficient-chaining}, which is only guaranteed to work with unit Lipschitz constants. However, a similar regret bound can be achieved for arbitrary (yet known) values of $L_f$ via a simple modification---see also \citep{bubeck2011lipschitz} for regret bounds optimized for $L_f$  in a stochastic and non-contextual bandit setting when $L_f$ is unknown.} w.r.t.\ $\rho_{\cX}$ and $\rho_{\cY}$. Namely, $\rho_{\cY}\big(f(x),f(x')\big) \leq \rho_{\cX}(x,x')$ for all $f\in\cF$ and all $x,x'\in\cX$.
We occasionally use the dot product notation $p_t\cdot\ell_t$ to indicate the expectation of $\ell_t$ according to law $p_t$. Finally, the set of all probability distributions over a finite set of $K$ elements is denoted by $\Delta(K)$.

\paragraph{Organization of the paper.}
The rest of the paper is organized as follows. In Section \ref{sec:related}, we give an overview of the related literature. Then, starting from the subsequent section, our results are presented in the order dictated by the amount of feedback available to the learner, from bandit feedback (Section \ref{sec:bandits}) to one-sided full information feedback (Section \ref{sec:sideInformation}) to full information feedback (Section \ref{sec:full-info}).

%% file: section-related.tex
%
Contextual online learning generalizes online convex optimization~\citep{hazan2016introduction} to nonconvex losses, nonparametric policies, and partial feedback. Papers about nonparametric online learning in full information include \citep{vovk2007competing,GaillardGerchinovitz2015} for the square loss, and \citep{HaMe-07colt-OnlineLearningPriorKnowledge,RaSr-15-OnlineNonparametricGeneralLoss} for general convex losses. In the bandit feedback model, earlier work on context-free bandits on metric spaces includes \citep{kleinberg04,ksu08}. The paper \citep{AuCBFrSc-02-NonStochasticBandits} introduces the Exp4 algorithm for nonstochastic contextual bandits when both the action space and the policy space are finite, and policies are maps from contexts to distributions over actions. Moreover, rather than observing the current context, the learner sees the output of each policy for that context. In the contextual bandit model of \citet{maillard2011adaptive}, context space and action space are finite, and the learner observes the current context while competing against the best policy among all functions mapping contexts to actions. Finally, a nonparametric bandit setting related to ours was studied by \citet{sliv14}. We refer the reader to the discussion after Theorem~\ref{th:flat-lip} for connections with our results.

Chaining \citep{dudley1967sizes} is a powerful technique to obtain tail bounds on the suprema of stochastic processes. In nonparametric online learning with full information feedback, chaining was used constructively by \citet{cesa1999prediction} to design an algorithm for linear losses, and nonconstructively by \citet{rakhlin2015online} to derive minimax rates for Lipschitz losses. Other notable examples of chaining are the stochastic bandit algorithms of \citet{contal2015optimization} and \citet{contal2016stochastic}. The constructive algorithmic chaining technique developed in this work is inspired by the nonparametric analysis of the full information setting of \citet{GaillardGerchinovitz2015}. However, their multi-variable EG algorithm heavily relies on convexity of losses and requires access to loss gradients. In order to cope with nonconvex losses and lack of gradient information, we develop a novel chaining approach based on a tree of hierarchical coverings of the policy class, where each internal tree node hosts a bandit algorithm.

In our nonstochastic online setting, chaining yields improved rates when the regret is decomposed into a sum of local regrets, each one scaling with the range of the local losses. However, deriving regret bounds that scale with the effective range of the losses is not always possible, as shown by \citet{Gerchinovitz2016} in the nonstochastic $K$-armed bandit setting. This result suggests that chaining might not be useful in online nonparametric learning when the feedback is bandit. However, as we show in this paper, algorithmic chaining does help improving the regret when the feedback is one-sided full information or full information. In full information, chaining-based algorithms deliver regret bounds that match (up to log factors) the nonconstructive bounds of~\citep{rakhlin2015online}.

In a different but interesting research thread on contextual bandits, the learner is confronted with the best within a finite (but large) class of policies over finitely-many actions, and is assumed to have access to this policy class through an optimization oracle for the offline full information problem. Relevant references include \citep{agarwaletal14,RK16,sks16}. The main concern is to devise (oracle-based) algorithms with small regret and requiring as few calls to the optimization oracle as possible.

%% file: section-bandits.tex
%
As a simple warmup exercise, we prove a known result ---see e.g., \citep{sliv14}. Namely, a regret bound for contextual bandits with Lipschitz policies and Lipschitz losses. ContextualExp3 (Algorithm~\ref{alg:ContextualExp3}) is a bandit version of the algorithm by~\citet{HaMe-07colt-OnlineLearningPriorKnowledge} and maintains a set of balls of fixed radius $\ve$ in the context space, where each ball hosts an instance of the Exp3 algorithm of \citet{AuCBFrSc-02-NonStochasticBandits}.\footnote{
Instead of Exp3 we could use INF \citep{audibert2010regret}, which enjoys a minimax optimal regret bound up to constant factors. This would avoid a polylog factor in $T$ in the bound. Since we do not optimize for polylog factors anyway, we opted for the better known algorithm.
}
At each round $t$, if a new incoming context $x_t \in \scX$ is not contained in any existing ball, then a new ball centered at $x_t$ is created, and a fresh instance of Exp3 is allocated to handle $x_t$. Otherwise, the Exp3 instance associated with the closest context so far w.r.t. $\rho_{\scX}$ is used to handle $x_t$. Each allocated Exp3 instance operates on the discretized action space $\cY_{\ve}$ whose size $K_{\ve}$ is at most $C_{\cY}\,\ve^{-p}$.
\begin{algorithm2e}[!ht]
{\bfseries Input:} Ball radius $\ve > 0$, $\ve$-covering $\cY_{\ve}$ of $\cY$ such that $|\cY_{\ve}| \le C_{\cY}\,\ve^{-p}$.\\
\For{$t=1,2,\dots$}{
\begin{enumerate}[topsep=0pt,parsep=0pt,itemsep=0pt]
	\item Get context $x_t \in \scX$;
	\item If $x_t$ does not belong to any existing ball, then create a new ball of radius $\ve$ centered on $x_t$, and allocate a fresh instance of Exp3;
	\item Let the active Exp3 instance be the instance allocated to the existing ball whose center $x_s$ is closest to $x_t$;
	\item Draw an action $\yhat_t$ using the active Exp3 instance;
	\item Get $\ell_t\big(\yhat_t\big)$ and use it to update the active Exp3 instance. 
\end{enumerate}}
\caption{ContextualExp3 (for bandit feedback)}
\label{alg:ContextualExp3}
\end{algorithm2e}
\ifarxiv
For completeness, the proof of the following theorem is provided in Appendix~\ref{as:folklore}.
\fi
\begin{theorem}\label{th:flat-lip}
Fix any any sequence $(x_1,\loss_1),(x_2,\loss_2),\dots$ of contexts $x_t\in\cX$ and $1$-Lipschitz loss functions $\ell_t : \cY \to [0,1]$. If ContextualExp3 is run in the bandit feedback model with parameter\footnote
{
Here and throughout, $T$ is assumed to be large enough so as to ensure $\epsilon \leq 1$.
} 
$\ve = (\ln T)^{\frac{2}{p+d+2}}\,T^{-\frac{1}{p+d+2}}$, then its regret $\Reg_T(\cF)$ with respect to the set $\cF$ of $1$-Lipschitz functions $f : \cX\to\cY$ satisfies
$
	\Reg_T(\cF)
=
	{\tilde \cO}\big(T^\frac{p+d+1}{p+d+2}\big)
$, where the ${\tilde \cO}$ notation hides factors polynomial in $C_{\cX}$ and $C_{\cY}$, and $\ln T$ factors.
\end{theorem}
%
%
%
%
A lower bound matching up to log factors the upper bound of Theorem~\ref{th:flat-lip} is contained in \citep{sliv14} ---see also \citep{lpp10} for earlier results in the same setting. However, our setting and his are subtly different: the adversary of \citet{sliv14} uses more general Lipschitz losses which, translated into our context, imply that the Lipschitz assumption is required to hold only for the composite function $\ell_t(f(\cdot))$, rather than the two functions $\ell_t$ and $f$ separately. Hence, being the adversary less constrained (and the comparison class wider), the lower bound contained in \citep{sliv14} does not seem to apply to our setting.
%
%
%

While we are unaware of a lower bound matching the upper bound in Theorem~\ref{th:flat-lip} when $\cF$ is the class of (global) Lipschitz functions and $d \geq 1$, in the noncontextual case ($d=0$), the lower bound $\Omega\big(T^{(p+1)/(p+2)}\big)$ proven by \citet{bubeck2011x,ksu08} shows that improvements on the dependence on $p$ are generally impossible. Yet,
%
%
%
%
the dependence on $p$ in the bound of Theorem~\ref{th:flat-lip} can be greatly improved in the special case when the Lipschitz losses are also convex. Assume $\cY$ is a convex and compact subset of $\R^p$. Then we use the same approach as in Theorem~\ref{th:flat-lip}, where the Exp3 algorithm hosted at each ball is replaced by an instance of the algorithm by~\citet{bubeck2016kernel}, run on the non-discretized action space $\cY$. The regret of the algorithm that replaces Exp3 is bounded by $\mathrm{poly}(p,\ln T)\sqrt{T}$. This immediately gives the following corollary.
\begin{corollary}
\label{cor:convex}
Fix any any sequence $(x_1,\loss_1),(x_2,\loss_2),\dots$ of contexts $x_t\in\cX$ and convex loss functions $\ell_t : \cY \to [0,1]$, where $\cY$ is a convex and compact subset of $\R^p$. Then there exists an algorithm for the bandit feedback model whose regret with respect to the set $\cF$ of $1$-Lipschitz functions satisfies $\Reg_T(\cF) \leq \mathrm{poly}(p,\ln T)T^{(d+1)/(d+2)}$, where $\mathrm{poly}$ is a polynomial function of its arguments.
\end{corollary}




%% file: section-sideinfo-introwarmup.tex

In this section we show that better nonparametric rates can be achieved in the one-sided full information setting, where the feedback is larger than the standard bandit feedback but smaller than the full information feedback. More precisely, we consider the same setting as in Section~\ref{sec:bandits} in the special case when the action space $\cY$ is $[0,1]$. We further assume that, after each play $\yhat_t \in \cY$, the learner can compute the loss $\ell_t(y)$ of any number of actions $y \geq \yhat_t$. This in contrast to observing only $\loss_t\big(\hat{y}_t\big)$, as in the standard bandit setting. We start with an important special case: maximizing the seller's revenue in a sequence of repeated second-price auctions. In Section~\ref{sec:chaining}, we use the chaining technique to design a general algorithm for arbitrary Lipschitz losses in the one-sided full information model. An efficient variant of this algorithm is obtained using a more involved construction in Section~\ref{sec:efficient-chaining}.

\subsection{Nonparametric second-price auctions}
\label{sec:auctions}
%
%
In online advertising, publishers sell their online ad space to advertisers through second-price auctions managed by ad exchanges. For each impression (ad display) created on the publisher's website, the ad exchange runs an auction on the fly. Empirical evidence~\citep{OS11} shows that an informed choice of the seller's reserve price, disqualifying any bid below it, can indeed have a significant impact on the revenue of the seller. Regret minimization in second-price auctions was studied by~\citet{CeGeMa-15-ReservePriceOptimization} in a non-contextual setting. They showed that, when buyers draw their bids i.i.d.\ from the same unknown distribution on $[0,1]$, there exists an efficient strategy for setting reserve prices such that the seller's regret is bounded by~$\tilde{\cO}\big(\sqrt{T}\big)$ with high probability with respect to the bid distribution. Here we extend those results to a nonstochastic and nonparametric contextual setting with nonstochastic bids, and prove a regret bound of order $T^{(d+1)/(d+2)}$ where $d$ is the context space dimension. This improves on the bound $T^{(d+2)/(d+3)}$ of Theorem~\ref{th:flat-lip} when $p=1$. 
\ifarxiv
As a byproduct, in Theorem~\ref{alg:exp3floor} (Appendix~\ref{app:exp3floor}) we prove the first $\tilde{\cO}\big(\sqrt{T}\big)$ regret bound for the seller in nonstochastic and noncontextual second-price auctions. 
\fi
Unlike~\citep{CeGeMa-15-ReservePriceOptimization}, where the feedback after each auction was ``strictly bandit'' (i.e., just the seller's revenue), here we assume the seller is also observing the highest bid together with the revenue. This richer feedback, which is key to proving our results, is made available by some ad exchanges such as AppNexus.

The seller's revenue in a second-price auction is computed as follows: if the reserve price $\hat{y}$ is not larger than the second-highest bid $b(2)$, then the item is sold to the highest bidder and the seller's revenue is $b(2)$. If $\hat{y}$ is between $b(2)$ and the highest bid $b(1)$, then the item is sold to the highest bidder but the seller's revenue is the reserve price. Finally, if $\hat{y}$ is bigger than $b(1)$, then the item is not sold and the seller's revenue is zero. Formally, the seller's revenue is $g\big(\hat{y},b(1),b(2)\big) = \max\big\{\hat{y},b(2)\big\}\indicator{\hat{y} \le b(1)}$. Note that the revenue only depends on the reserve price $\hat{y}$ and on the two highest bids $b(1) \ge b(2)$, which ---by assumption--- belong all to the unit interval $[0,1]$.

In the online contextual version of the problem, unknown sequences of contexts $x_1,x_2,\ldots\in\cX$ and bids are fixed beforehand (in the case of online advertising, the context could be public information about the targeted customers). At the beginning of each auction $t=1,2,\dots$, the seller observes context $x_t$ and computes a reserve price $\hat{y}_t \in [0,1]$. Then, bids $b_t(1),b_t(2)$ are collected by the auctioneer, and the seller (which is not the same as the auctioneer) observes his revenue $g_t\big(\hat{y}_t\big)=g\big(\hat{y}_t,b_t(1),b_t(2)\big)$, together with the highest bid $b_t(1)$. Crucially, knowing $g_t(\hat y_t)$ and $b_t(1)$ allows to compute $g_t(y)$ for all $y \geq \hat y_t$. For technical reasons, we use losses $\loss_t\big(\hat{y}_t\big) = 1 - g_t\big(\hat{y}_t\big)$ instead of revenues, see Figure~\ref{f:second_price_loss} for a pictorial representation.
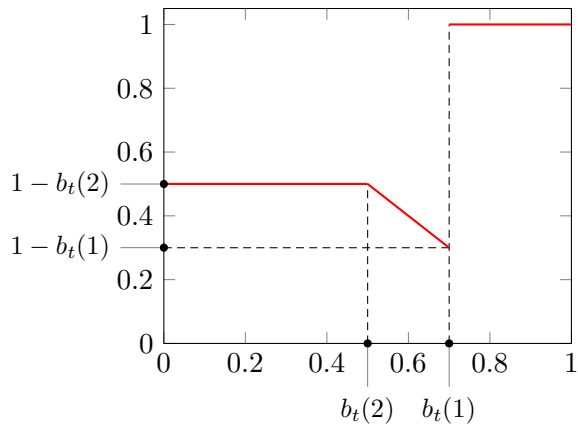
\begin{figure}[ht!]
\begin{center}
\tikzset{
every pin/.style={font=\small},
small dot/.style={fill=black,circle,scale=0.3}
}
\begin{tikzpicture}
\begin{axis}[
clip=false,
xmin=0, xmax=1,
ymin=0, ymax=1.05,
yticklabel pos=left,
ytick={0,0.2,0.4,0.6,0.8,1}
]
\addplot[no markers,densely dashed] coordinates { (0.5,0) (0.5,0.5)};
\addplot[no markers,densely dashed] coordinates { (0.7,0) (0.7,1.0) };
\addplot[no markers,densely dashed] coordinates { (0,0.3) (0.7,0.3) };
\addplot[domain=0:0.5,thick,color=red] {0.5};
\addplot[domain=0.5:0.7,thick,color=red] {1-x};
\addplot[domain=0.7:1,thick,color=red] {1};
\node[small dot,pin=-90:{$b_t(2)$}] at (xticklabel* cs:0.5) {};
\node[small dot,pin=-90:{$b_t(1)$}] at (xticklabel* cs:0.7) {};
\node[small dot,pin=180:{$1-b_t(2)$}] at (yticklabel* cs:0.476) {};
\node[small dot,pin=180:{$1-b_t(1)$}] at (yticklabel* cs:0.286) {};
\end{axis}
\end{tikzpicture}
\vspace*{-5mm}
\end{center}
\caption{\label{f:second_price_loss}
The loss function $\ell_t(\hat y_t) = 1-\max\{\hat y_t,b_t(2)\}\indicator{\hat{y}_t \le b_t(1)}$ when $b_t(1) = 0.7$ and $b_t(2) = 0.5$.
}
\end{figure}

Remarkably, the loss functions $\ell_t$ are not Lipschitz and not even continuous, and so this problem falls outside the scope of standard results for contextual bandits. Instead, the losses only satisfy the semi-Lipschitz property $\ell_t(y+\delta) \ge \ell_t(y) - \delta$ for all $0 \leq y \leq y+\delta \leq 1$. We now state a bound on the regret $\Reg_T(\cF)$ with respect to any class $\cF$ of Lipschitz functions $f : \cX \to [0,1]$. The algorithm that achieves this bound is ContextualRTB (where RTB stands for Real Time Bidding ---see Algorithm
\ifarxiv
\ref{alg:exp3floor} in Appendix~\ref{app:exp3floor}),
\fi a variant of ContextualExp3 (Algorithm \ref{alg:ContextualExp3}), where each ball hosts an instance of Exp3-RTB, instead of Exp3. The proof is given in
\ifarxiv 
Appendix~\ref{as:proofsecondpriceflat}.
\fi
%
%
\begin{theorem}\label{t:secondpriceflat}
Fix any sequence of contexts $x_t \in \cX$ and bid pairs $0 \leq b_t(2) \leq b_t(1) \leq 1$ for $t \geq 1$.
If ContextualRTB is run with parameter $\smash{\epsilon = T^{-\frac{1}{d+2}}}$ and Exp3-RTB is tuned with parameter $\gamma=\ve$, then the regret with respect to any class of $1$-Lipschitz functions $\smash{f : \cX \to [0,1]}$ satisfies
$
	\smash{\Reg_T(\cF) = {\tilde \cO}\big( T^{\frac{d+1}{d+2}}\big)}
$,
where $d$ is the dimension of $\cX$ and the ${\tilde \cO}$ notation hides constants and $\ln T$ factors.
\end{theorem}
ContextualRTB and ContextualExp3 of Section~\ref{sec:bandits} can be modified so to avoid knowing the horizon $T$ and so that the dimension $d$ of the context space is replaced in the bound by the (unknown, and possibly much smaller) dimension of the set of contexts actually occurring in the sequence chosen by the adversary. This modification involves using a time-varying radius $\ve$ and a doubling trick to check when the current guess for the dimension is violated by the current number of balls. The omitted proof of this statement goes along the lines of the proof in \cite[Theorem~1]{RoOrCB-15-NonparametricClassification}.

%% file: section-sideinfo-chaining.tex

\subsection{Chaining the bandits}
\label{sec:chaining}
We now show that whenever the richer feedback structure ---i.e., the learner can compute the loss $\ell_t(y)$ of any number of actions $y \geq \yhat_t$--- is combined with Lipschitz losses (rather than just semi-Lipschitz), then an improved regret bound $T^{d/(d+1)}$ can be derived. The key technical idea enabling this improvement is the application of the chaining technique to a hierarchical covering of the policy space (as opposed to the flat covering of the context space used in both Section \ref{sec:bandits} and Section~\ref{sec:auctions}). We start with a computationally inefficient algorithm that works for arbitrary policy classes $\cF$ (not only Lipschitz) and is easier to understand. In Section~\ref{sec:efficient-chaining} we derive an efficient variant for $\cF$ that are Lipschitz. In this case we obtain even better regret bounds via a penalization trick.

A way of understanding the chaining approach is to view the hierarchical covering of the policy class $\cF$ as a tree whose nodes are functions in $\cF$, and where the nodes at each depth $m$ define a $(2^{-m})$-covering of $\cF$. The tree represents any function $f^* \in \cF$ (e.g., the function with the smallest cumulative loss) by a unique path (or chain) $f_0 \to f_1 \to\dots\to f_M \to f^*$, where $f_0$ is the root and $f_M$ is the function best approximating $f^*$ in the cover at the largest available depth $M$. By relying on this representation, we control the regret against any function in $\cF$ by running an instance of an online bandit algorithm $A$ on each node of the tree. The instance $A_f$ at node $f$ uses the predictions of the instances running on the nodes that are children of $f$ as expert advice. The action drawn by instance $A_0$ running on the root node is the output of the tree. For any given sequence of pairs $(x_t,\loss_t)$ of contexts and losses, the regret against $f^*$ with path $f_0 \to f_1 \to\dots\to f_M \to f^*$ can then be written (ignoring some constants) as
\[	\sum_{t=1}^T \Big( \E\left[\loss_t\big(A_0(x_t)\big)\right] - \loss_t\big(f^*(x_t)\big) \Big) \leq \sum_{m=0}^{M-1}\E\!\left[ \sum_{t=1}^T \Big( \loss_t\big(A_m(x_t)\big) - \loss_t\big(A_{m+1}(x_t)\big) \Big) \right] + 2^{-M}T
\]
where $A_m$ is the instance running on node $f_m$ for $m = 0, \ldots, M-1$ and $A_M \equiv f_M$. The last term $2^{-M}T$ accounts for the cost of approximating $f^*$ with the closest function $f_M$ in a $(2^{-M})$-cover of $\cF$ under suitable Lipschitz assumptions. The outer sum in the right-hand side of the above display can be viewed as a sum of $M$ regrets, where the $m$-th term in the sum is the regret of $A_m$ against the instances running on the children of the node hosting $A_m$. Since we face an expert learning problem in a partial information setting, the Exp4 algorithm of~\citet{AuCBFrSc-02-NonStochasticBandits} is a natural choice for the learner $A$. However, a first issue to consider is that we are using $A_0$ to draw actions in the bandit problem, and so the other Exp4 instances receive loss estimates that are based on the distribution used by $A_0$ rather than being based on their own distributions. A second issue is that our regret decomposition crucially relies on the fact that each instance $A_m$ only competes (in the sense of regret) against functions $f$ at the leaves of the subtree rooted at the node where $A_m$ runs. By construction, these functions at the leaves are roughly $(2^{-m})$-close to each other and ---by Lipschitzness--- so are their losses. As a consequence, the regret of $A_m$ should scale with the true loss range $2^{-m}$. Via an appropriate modification of the original Exp4 algorithm, we manage to address both these issues. In particular, in order to make the regret dependent on the loss range, we heavily rely on the one-sided full information model assumed in this section. Finally, the hierarchical covering requires that losses be Lipschitz, rather than just semi-Lipschitz as in the application of Subsection~\ref{sec:auctions}, which uses a simpler flat covering.

Fix any class $\cF$ of functions $f : \cX \to [0,1]$. Let us introduce the sup norm
\begin{equation}
\label{eq:supnorm}
	\norm{f-g}_\infty = \sup_{x \in \cX} \big|f(x)-g(x)\big|~.
\end{equation}
We denote by $\cN_\infty\big(\cF,\epsilon\big)$ the cardinality of the smallest $\epsilon$-cover of $\cF$ w.r.t. the sup norm.
Throughout this section, our only assumption on $\cF$ is that $\cN_\infty\big(\cF,\epsilon\big) < +\infty$ for all $\epsilon > 0$ (this property is known as \emph{total boundedness}). In Section~\ref{sec:efficient-chaining} we will focus on the special case $\cF = \big\{f:[0,1]^d \to [0,1] : \textrm{$f$ is $1$-Lipschitz}\bigr\}$ to derive an efficient version of our algorithm.

We now define a tree $\cT_{\cF}$ of depth $M$, whose nodes are labeled by functions in the class $\cF$, so that functions corresponding to nodes with a close common ancestor are close to one another according to the sup norm~(\ref{eq:supnorm}). For all $m=0,1,\ldots,M$, let $\cF_m$ be a $(2^{-m})$-covering of $\cF$ in sup norm with minimal cardinality $N_m = \cN_\infty(\cF,2^{-m})$. Since the diameter of $(\cF,\norm{\cdot}_{\infty})$ is bounded by $1$, we have $N_0 = 1$ and $\cF_0 = \{f_0\}$ for some $f_0 \in \cF$. For each $m=0,1,\dots,M$ and for every $f_v \in \cF_m$ we have a node $v$ in $\cT_{\cF}$ at depth $m$. The parent of a node $w$ at depth $m+1$ is some node $v$ at depth~$m$ such that
\[
	v \in \argmin_{v'\,:\,\depth(v')=m} \norm{f_{v'} - f_w}_{\infty} \quad \textrm{(ties broken arbitrarily)}
\]
and we say that $w$ is a child of $v$. Let $\cL$ be the set of all the leaves of $\cT_{\cF}$, $\cL_v$ be the set of all the leaves under $v \in \cT_{\cF}$ (i.e., the leaves of the subtree rooted at $v$), and $\cC_v$ be the set of children of $v \in \cT_{\cF}$.

\input{tree-figure}

Our new bandit algorithm HierExp4 (Algorithm~\ref{alg:HierarchicalExp4} below) is a hierarchical composition of instances of Exp4 on the tree $\cT_{\cF}$ constructed above (see Figure~\ref{fig:tree}). Let $K=2^M$ and $\cK = \{y_1,\dots,y_K\}$, where $y_k = 2^{-M}(k-1)$ for $k=1,\dots,2^M$, be our discretization of the action space $\cY = [0,1]$. At every round $t$, after observing context $x_t \in \cX$, each leaf $v \in \cL$ recommends the best approximation of $f_v(x_t)$ in $\cK$,
$
i_t(v) \in \argmin_{i=1,\dots,K} \big|y_i - f_v(x_t)\big|
$.
Therefore, the leaves $v \in \cL$ correspond to deterministic strategies $t \mapsto i_t(v)$, and we will find it convenient to view a set of leaves $\cL$ as the set of actions played by those leaves at time $t$. Each internal node $v \in \cT_{\cF} \setminus \cL$ runs an instance of Exp4 using the children of $v$ as experts. More precisely, we use a variant of Exp4 
\ifarxiv
(see Algorithm~\ref{alg:exp4} in Appendix~\ref{sec:Exp4range})
\fi which adapts to the effective range of the losses. Let $\automEWA_v$ be the instance of the Exp4 variant run on node $v$. At each time $t$, this instance updates a distribution $q_t(v,\cdot) \in \Delta(|\cC_v|)$ over experts in $\cC_v$ and a distribution $p_t(v,\cdot) \in \Delta(K)$ over actions in $\cK$ defined by $p_t(v,i) = \sum_{w \in \cC_v} q_t(v,w) p_t(w,i)$.

Let $v_0$ be the root of $\cT_{\cF}$. The prediction of HierExp4 at time $t$ is $\hat{y}_t = y_{I_t} \in \cK$, where $I_t$ is drawn according to a mixture of $p_t(v_0,\cdot)$ and a unit mass on the minimal action $y_1 \in \cK$.



For each $v \in \cT_{\cF}\setminus\cL$, let
$\cK_t(v) = \theset{i}{(\exists w \in \cC_v)\; p_t(w,i) > 0}$ and $j_t(v) = \max\cK_t(v)$.
Note that $\hat{\loss}_t(v,i)$ in~\eqref{eq:estimatedloss} has to be explicitly computed only for those actions $i$ such that $i \geq I_t$ and $i \in \cK_t(v)$. This is because $\hat{\loss}_t(v,i)$ is needed for the computation of $\tilde{\ell}_t(v,w)$ only when $p_t(w,i) > 0$. Therefore, whenever $\hat{\loss}_t(v,i)$ has to be computed for some $i$, then $I_t \leq i \leq \max \cK_t(v) = j_t(v)$, so that $\loss_t\big(y_{j_t(v)}\big)$ is observed and $\hat{\loss}_t(v,i)$ is well defined.

\begin{algorithm2e}[!ht]
\SetKwInOut{Input}{Input}
\SetKwInOut{Init}{Initialization}

\Input{Tree $\cT_{\cF}$ with root $v_0$ and leaves $\cL$, exploration parameter $\gamma \in (0,1)$, learning rate sequences $\eta_1(v) \geq \eta_2(v) \geq\cdots > 0$ for $v \in \cT_{\cF} \setminus \cL$.
}
\Init{
Set $q_1(v,\cdot)$ to the uniform distribution in $\Delta(|\cC_v|)$ for every $v \in \cT_{\cF} \setminus \cL$.
} 

\For{$t=1,2,\ldots$}{
\begin{enumerate}[topsep=0pt,parsep=0pt,itemsep=0pt]
	\item Get context $x_t \in \cX$;
	\item Set $p_t(v,i) = \indicator{i=i_t(v)}$ for all $i\in \cK$ and for all $v \in \cL$;
	\item Set $p_t(v,i) = q_t(v,\cdot) \cdot p_t(\cdot,i)$ for all $i\in \cK$ and for all $v \in \cT_{\cF} \setminus \cL$;
	\item Draw $I_t \sim p^*_t$ and play $\hat{y}_t = y_{I_t}$, where $p^*_t(i) = (1-\gamma) p_t(v_0,i) + \gamma\indicator{i = 1}$ for all $i \in \cK$;
	\item Observe $\ell_t(y)$ for all $y \geq y_{I_t}$;
	\item For every $v \in \cT_{\cF} \setminus \cL$ and for every $i \in \cK_t(v)$ compute
	\begin{equation}
		\label{eq:estimatedloss}
		\hat{\ell}_t(v,i) = \frac{\ell_t(y_i) - \ell_t\big(y_{j_t(v)}\big)}{\sum_{k=1}^i p^*_t(k)}\indicator{I_t \leq i} \,,
	\end{equation}
	where $\cK_t(v) = \theset{i}{(\exists w \in \cC_v)\; p_t(w,i) > 0}$ and $j_t(v) = \max\cK_t(v)$.
	\item For each $v \in \cT_{\cF} \setminus \cL$ and for each $w \in \cC_v$ compute the expert loss
	$
		\tilde{\ell}_t(v,w) = p_t(w,\cdot)\cdot\hat{\ell}_t(v,\cdot)
	$ and perform the update
	\begin{equation}
	\label{eq:EWAweights}
		q_{t+1}(v,w) = \frac{\exp\!\left(-\eta_{t+1}(v) \sum_{s=1}^t \tilde{\loss}_s(v,w) \right)}{\sum_{w'\in\cC_v} \exp\!\left(-\eta_{t+1}(v) \sum_{s=1}^t \tilde{\loss}_s(v,w') \right)}
	\end{equation}
\end{enumerate}
}
\caption{
\label{alg:HierarchicalExp4}
HierExp4 (for one-sided full information feedback)
}
\end{algorithm2e}
Next, we show that the regret of HierExp4 is at most of the order of $T^{d/(d+1)}$, which improves on the rate $T^{(d+1)/(d+2)}$ obtained in Section~\ref{sec:auctions} without using chaining. The required proofs are contained in 
\ifarxiv
Appendix~\ref{as:missingproofs}.
\fi
\begin{theorem}
\label{thm:HierarchicalExp4}
Fix any class $\cF$ of functions $f : \cX\to [0,1]$ and any sequence $(x_1,\loss_1),(x_2,\loss_2),\dots$ of contexts $x_t\in\cX$ and $1$-Lipschitz loss functions $\ell_t : [0,1] \to [0,1]$. Assume the HierExp4 (Algorithm~\ref{alg:HierarchicalExp4}) is run with one-sided full information feedback using tree $\cT_{\cF}$ of depth $M = \lfloor \ln_2(1/\gamma) \rfloor$. Moreover, the learning rate $\eta_t(v)$ used at each node $v$ at depth $m=0,\dots,M-1$ is given by
\ifarxiv
~(\ref{eq:adarate}) in Appendix~\ref{sec:Exp4range},
\fi 
where ${\tilde V}_{t-1}(v)$ is the cumulative variance of $\tilde{\loss}_s(v,\cdot)$ according to $q_s(v,\cdot)$ up to time $s=t-1$. Then for all $T \geq 1$ the regret satisfies\\[-0.3cm]
\begin{align*}
	\Reg_T(\cF) & \leq
	5 \gamma T + 2^7 \int_{\gamma/2}^{1/2} \left(\sqrt{\frac{T}{\gamma} \ln\cN_\infty(\cF,\varepsilon)} + \frac{1}{\gamma} \Bigl(\ln\cN_\infty(\cF,\varepsilon) + 1 \Bigr) \right) \dd \varepsilon~.
\end{align*}
\end{theorem}
In particular, if $\cX \equiv [0,1]^d$ is endowed with the sup norm $\rho_\cX(x,x')=\|x-x'\|_\infty$, then the set $\cF$ of all $1$-Lipschitz functions from $\cX$ to $[0,1]$ satisfies $\ln \cN_\infty(\cF,\epsilon) \lesssim \epsilon^{-d}$. Theorem~\ref{thm:HierarchicalExp4} thus entails the following corollary.
%
\begin{corollary}
\label{cor:HierearchicalExp4}
Under the assumptions of Theorem~\ref{thm:HierarchicalExp4}, if $\cF$ is the set of all $1$-Lipschitz functions $f : [0,1]^d \to [0,1]$, then the regret of HierExp4 satisfies\\[-0.2cm]
\[
	\Reg_T(\cF)
=
	\left\{ 
		\begin{array}{ll}
			\cO\big(T^{2/3}\big) & \text{if $d=1$} \\
			\cO\big(T^{2/3}(\ln T)^{2/3}\big) & \text{if $d=2$} \\
			\cO\big(T^{d/(d+1)}\big) & \text{if $d\geq 3$}
		\end{array} 
	\right.
\]
where the last inequality is obtained by optimizing the choice of $\gamma$ for the different choices of $d$.
\end{corollary}
%
%
%

%% file: tree-figure.tex

\begin{figure}[!th]
\begin{center}
\resizebox{\textwidth}{!}{
\begin{tikzpicture}[level/.style={sibling distance=60mm/#1,circle,inner sep=0pt, minimum size=1cm}]
\node [circle,draw,inner sep=0pt, minimum size=1cm,fill=blue!20] (u) {$u$}
  child {node [circle,draw,fill=blue!20] (a) {}
    child {node [circle,draw,fill=blue!20] (b) {}
        edge from parent[draw=none]
    }
    child {node [circle,draw,fill=blue!20] (g) {}
        edge from parent[draw=none]
    }
  }
  child {node [circle,draw,fill=blue!20] (v) {$v$}
    child {node [circle,draw,fill=blue!20] (k) {}
        edge from parent[draw=none]
    }
    child {node [solid,circle,draw,fill=blue!20] (w) {$w$} 
        edge from parent[draw=none]
   }
};
\path (a) -- (v) node [midway] {\dots};
\path (v) -- (w) node [midway] {\rotatebox{135}{\dots}};
\path (v) -- (k) node [midway] {\rotatebox{45}{\dots}};
\path (a) -- (g) node [midway] {\rotatebox{135}{\dots}};
\path (a) -- (b) node [midway] {\rotatebox{45}{\dots}};
\path (b) -- (g) node [midway] {\dots};
\path (g) -- (k) node [midway] {\dots};
\path (k) -- (w) node [midway] {\dots};

\draw (u.south) node[below,color=blue!50]{Exp4};
\draw (a.south) node[below,color=blue!50]{Exp4};
\draw (v.south) node[below,color=blue!50]{Exp4};

\tikzset{
    position label/.style={
       below = 3pt,
       text height = 1.5ex,
       text depth = 1ex
    },
   brace/.style={
     decoration={brace, mirror},
     decorate
   }
}

\draw [brace,decoration={raise=1ex}] (k.south) -- (w.south) node [position label, pos=0.5,yshift=-1ex] {Leaves under $v \ = \ \cL_v$};

 \begin{scope}[every node/.style={right}]
   \path (u  -| w) ++(5mm,0) node {} ++(5mm,0) node {Level $m$ \qquad  $\leadsto 2^{-m}$ covering of $\cF$};
   \path (v  -| w) ++(5mm,0) node {} ++(5mm,0) node {Level $m+1$ \ $\leadsto 2^{-(m+1)}$ covering};
   \path (w -| w) ++(5mm,0) node {} ++(5mm,0) node {Level $M$ ($=$leaves $\cL$) $\leadsto 2^{-M}$ covering};
 \end{scope}

\end{tikzpicture}}
\vspace*{-1cm}
\end{center}
\caption{Hierarchical covering of the function space (used in Algorithm~\ref{alg:HierarchicalExp4}).}
\label{fig:tree}
\end{figure}
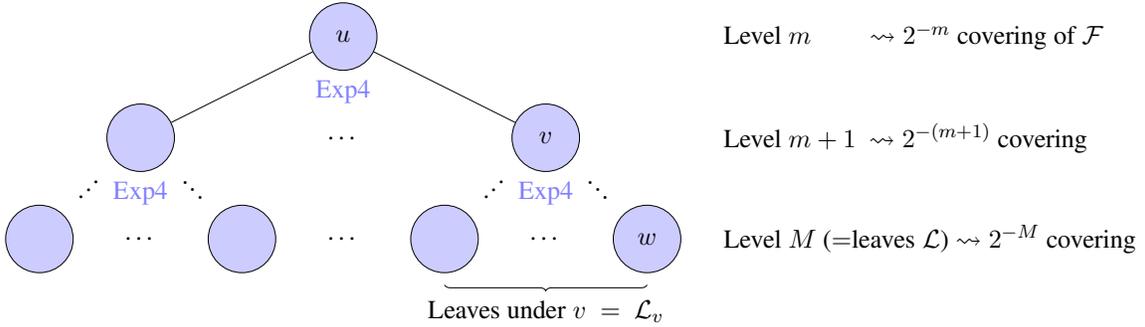

%% file: subsection-efficient-chaining.tex

\subsection{Efficient chaining}
\label{sec:efficient-chaining}
Though very general, HierExp4 (Algorithm~\ref{alg:HierarchicalExp4}) may be very inefficient. For example, when $\cF$ is the set of all $1$-Lipschitz functions from $[0,1]^d$ to $[0,1]$, a direct implementation of HierExp4 would require $\exp\bigl(\mathrm{poly}(T)\bigr)$ weight updates at every round.
In this section we tackle the special case when $\cF$ is the class of all $1$-Lipschitz functions $f : [0,1]^d  \to [0,1]$ w.r.t. the sup norm on $[0,1]^d$ (for simplicity). We construct an ad-hoc hierarchical covering of $\cF$ and define a variant of HierExp4 whose running time at every round is polynomial in $T$. We rely on a well-known wavelet-like approximation technique which was used earlier ---see, e.g., \citep{GaillardGerchinovitz2015}--- for online nonparametric regression with full information feedback. However, we replace their multi-variable Exponentiated Gradient algorithm, which requires convex losses and gradient information, with a more involved chaining algorithm that still enjoys a polynomial running time. The definitions of our covering tree $\cT_{\cF}^{\star}$ and of our algorithm $\textrm{HierExp4}^\star$, as well as the proof of the following regret bound, can be found in 
\ifarxiv
Appendix~\ref{as:proofHierExp4star}.
\fi
The exact value of $c_T$ (depending at most logarithmically on $T$) is also provided there.
\vspace{-0.3cm}
\begin{theorem}
\label{thm:HierExp4star}
Let $\cF$ be the set of all $1$-Lipschitz functions $f : [0,1]^d \to [0,1]$ w.r.t. the sup norm on $[0,1]^d$. Consider $T \geq 3$
and any sequence $(x_1,\loss_1),\ldots,(x_T,\loss_T)$ of contexts $x_t\in [0,1]^d$ and $1$-Lipschitz loss functions $\ell_t : [0,1] \to [0,1]$. Assume $\textrm{HierExp4}^\star$ 
\ifarxiv
(Algorithm~\ref{alg:HierExp4star} in Appendix~\ref{as:proofHierExp4star}) 
\fi
is run with one-sided full information feedback using tree $\cT^{\star}_{\cF}$ of depth $M = \lceil\ln_2(1/\gamma)\rceil$, exploration parameter $\gamma = T^{-1/2}(\ln T)^{-1}\indicator{d=1}+T^{-1/(d+2/3)}\indicator{d > 1}$, learning rate $\eta_m = c_T 2^{m(\frac{d}{4}+1)} \gamma^{\frac{1}{2}}T^{-\frac{1}{4}}$, and penalization $\smash{\alpha_m = \sum_{j=m+1}^M 2^{4-2j}\eta_j}$ for $m=0,\dots,M-1$. Then the regret satisfies\\[-0.2cm]
\[
	\Reg_T(\cF)
	= \left\{\begin{array}{ll}
		\mathcal{O}\big(\sqrt{T \ln T}\big) & \textrm{if $d = 1$,} \\
		\mathcal{O}\big(T^{\frac{d-1/3}{d+2/3}}  (\ln T)^{3/2}\big) & \textrm{if $d \geq 2$.}
	\end{array} \right.
\]
Moreover, the running time at every round is $\mathcal{O}\bigl(T^a\bigr)$ with $a=(1 + \ln_2 3)/(d+2/3)$.
\end{theorem}
The above result improves on Corollary~\ref{cor:HierearchicalExp4} in two ways. First, as we said, the running time is now polynomial in $T$, contrary to what could be obtained via a direct implementation of HierExp4. Second, when $d \geq 2$, the regret bound is of order $T^{(d-1/3)/(d+2/3)}$, improving on the rate $T^{d/(d+1)}$ from Corollary~\ref{cor:HierearchicalExp4}. Remarkably, Theorem~\ref{thm:HierExp4star} also yields a regret of $\tilde{\mathcal{O}}(\sqrt{T})$ for nonparametric bandits with one-sided full information feedback in dimension $d=1$. The improvement on the rates compared to HierExp4 is possible because we use a variant of Exp4 with \emph{penalized} loss estimates. This allows for a careful hierarchical control of the variance terms inspired by 
\ifarxiv
our analysis of Exp3-RTB in Appendix~\ref{app:exp3floor}.
\fi

Note that the time complexity decreases as the dimension $d$ increases. Indeed, when $d$ increases the regret gets worse but, at the same time, the size of the discretized action space and the number of layers in our wavelet-like approximation can be both set to smaller values.

%% file: section-fullinfo.tex

In this section we apply the machinery developed in Section~\ref{sec:chaining} to the full information setting, where after each round $t$ the learner can compute the loss $\loss_t(y)$ of any number of actions $y\in\cY$. We obtain the first explicit algorithm achieving, up to logarithmic factors, the minimax regret rate $T^{(d-1)/d}$ for all classes of Lipschitz functions, where $d$ is the dimension of the context space. This achieves the same upper bound as the one proven by \citet{rakhlin2015online} in a nonconstructive manner, and matches the lower bound of \citet{HaMe-07colt-OnlineLearningPriorKnowledge}. Our approach generalizes the approach of \cite{GaillardGerchinovitz2015} to nonconvex Lipschitz losses.
We consider a full information variant of HierExp4 (Algorithm~\ref{alg:HierarchicalExp4}, Section~\ref{sec:chaining}), where ---using the same notation as in Section~\ref{sec:chaining}--- the Exp4 instances running on the nodes of the tree $\cT_{\cF}$ are replaced by instances of Hedge ---e.g., \citep{Bubeck2012}. Note that, due to the full information assumption, the new algorithm, called HierHedge, observes losses at all leaves $v \in \cL$. As a consequence, no exploration is needed and so we can set $\gamma = 0$. For the same reason, the estimated loss vectors defined in~\eqref{eq:estimatedloss} can be replaced with the true loss vectors, $\ell_t$. See 
\ifarxiv
Algorithm~\ref{alg:HierarchicalEWA} in Appendix~\ref{as:proofHierarchicalEWA}
for a definition of HierHedge. The same appendix also contains a proof of the next result.
\fi
\begin{theorem} 
\label{thm:HierarchicalEWA}
Fix any class $\cF$ of functions $f : \cX\to\cY$ and any sequence $(x_1,\loss_1),(x_2,\loss_2),\dots$ of contexts $x_t\in\cX$ and $1$-Lipschitz loss functions $\ell_t : \cY \to [0,1]$. Assume HierHedge 
\ifarxiv(Algorithm~\ref{alg:HierarchicalEWA} in Appendix~\ref{as:proofHierarchicalEWA}) \fi is run with full information feedback on the tree $\cT_{\cF}$ of depth $M = \lfloor \ln_2(1/\ve) \rfloor$ with action set $\cY_{\ve}$ for $\ve > 0$. Moreover, the learning rate $\eta_t(v)$ used at each node $v$ at depth $m=0,\dots,M-1$ is given by~(\ref{eq:adarate}) \ifarxiv in Appendix~\ref{sec:Exp4range}, with $E = 2^{-m+3}$, $N = |\cC_v|$, and ${\tilde V}_{t-1}$ being the cumulative variance of $\tilde{\loss}_s$ according to $q_s(v,\cdot)$ up to time $s=t-1$\fi. Then for all $T \geq 1$ the regret satisfies
\[
	\Reg_T(\cF)
\leq
	5\ve T + 2^7 \int_{\ve/2}^{1/2} \left( 2\sqrt{T\ln\cN_{\infty}(\cF,x)} + \ln\cN_{\infty}(\cF,x) \right) \dd x~.
\]
In particular, if $d \geq 3$ and $\cF$ is the set of $1$-Lipschitz functions $f : [0,1]^d \to [0,1]^p$, where $[0,1]^d$ and $[0,1]^p$ are endowed with their sup norms, the choice $\epsilon = (p/T)^{1/d}$ yields
$
\Reg_T(\cF) = \tilde{O}\big(T^{(d-1)/d}\big)
$,
while for $1\leq d\leq 2$ the regret is of order $\sqrt{pT}$, ignoring logarithmic factors.
\end{theorem}
When using the sup norms, the dimension $p$ of the action space only appears as a multiplicative factor $p^{1/d}$ in the regret bound for Lipschitz functions. Note also that an efficient version of HierHedge for Lipschitz functions can be derived along the same lines as the construction in Section~\ref{sec:efficient-chaining}. 

%% file: appendix-lemmas.tex
\section{A Useful Lemma}
\label{sec:EWA-nonnegative}
The following is a standard result in the analysis of Hedge with variable learning rate ---see, e.g., the analysis of \citep[Theorem~3.2]{Bubeck2012}. Recall that, when run on a loss sequence $\ell_1,\ell_2,\dots$ Hedge computes distributions $p_1,p_2,\dots$ where $p_1$ is uniform and $p_t(i)$ is proportional to $\exp\big(-\eta_{t}\sum_{s=1}^{t-1} \ell_s(i) \big)$ for all actions $i$.
\begin{lemma} 
\label{lem:boundEWA-appe}
The sequence $p_1,p_2,\dots$ of probability distributions computed by Hedge when run on $K$ actions with learning rates $\eta_1 \geq \eta_2 \geq \cdots > 0$ and losses $\ell_t(k) \geq 0$ satisfies
\[
	\sum_{t=1}^T \sum_{i=1}^K p_t(i) \ell_t(i) - \min_{k=1,\dots,K} \sum_{t=1}^T \ell_t(k)
\leq
	\frac{\ln K}{\eta_T} + \frac{1}{2} \sum_{t=1}^T \eta_t\sum_{i=1}^K p_t(i) \ell_t(i)^2~.
\]
\end{lemma}

\section{Proof of Theorem \ref{th:flat-lip}}\label{as:folklore}
\begin{proof}
%
Let $N_t$ be the number of balls created by the ContextualExp algorithm after $t$ rounds, $B_s$ be the $s$-th ball so created, being $\xbar_s$ its center ($\xbar_s$ is some past context), and $T_s$ be the subsequence of rounds $t$ such that $x_t$ is handled by the $s$-th ball. Notice that $N_t$ and $T_s$ are deterministic quantities, since the $x_t$'s are generated obliviously.
Since $f$ is $1$-Lipschitz and $\loss_t$ is also $1$-Lipschitz, for all $x_t$ handled by $B_s$, we can write
$
	\big|\loss_t\big(f(x_t)\big) - \loss_t\big(f(\xbar_s)\big)\big| \le \ve\,.
$
Now fix any $1$-Lipschitz policy $f$. For each $s=1,\dots,N_T$, there exists $\ybar_s \in \cY_{\ve}$ such that $\rho_{\cY}\big(\ybar_s,f(\xbar_s)\big) \le \ve$. Then we can write
\begin{align}
	\E\left[\sum_{t=1}^T \loss_t(\yhat_t)\right] &- \sum_{t=1}^T \loss_t\big(f(x_t)\big)
\\ &=
	\sum_{s=1}^{N_T} \sum_{t \in T_s} \Big( \E\big[\loss_t(\yhat_t)\big] - \loss_t\big(f(x_t)\big) \Big)
\notag\\ &=
	\sum_{s=1}^{N_T} \sum_{t \in T_s} \Big( \E\big[\loss_t(\yhat_t)\big] - \loss_t\big(\ybar_s\big) + \loss_t\big(\ybar_s\big) - \loss_t(f(\xbar_s)) + \loss_t(f(\xbar_s)) - \loss_t(f(x_t)) \Big)
\notag\\ &\le
	\sum_{s=1}^{N_T} \sum_{t \in T_s} \Big( \E\big[\loss_t(\yhat_t)\big] - \loss_t\big(\ybar_s\big) \Big)
+ 2T\ve\,.\label{e:intermediate}
\end{align}
We now apply to each $s = 1, \ldots, N_T$ the standard regret bound of Exp3 with learning rate $\eta$ (e.g., \citep{Bubeck2012}) w.r.t. the best action $\ybar_s$. This yields
\[
\sum_{t \in T_s} \Big( \E\big[\loss_t(\yhat_t)\big] - \loss_t\big(\ybar_s\big) \Big) \leq \frac{\ln K_{\ve}}{\eta} + \frac{\eta}{2}\,|T_s|\,K_{\ve}\,,
\]
implying
\[
\sum_{s=1}^{N_T}\sum_{t \in T_s} \Big( \E\big[\loss_t(\yhat_t)\big] - \loss_t\big(\ybar_s\big) \Big) \leq \frac{N_T \ln K_{\ve}}{\eta} + \frac{\eta}{2}\,T\,K_{\ve}\,.
\]
Combining with (\ref{e:intermediate}), setting $\eta = \sqrt{\frac{2 N_T\,\ln K_{\ve}}{T\,K_{\ve}}}$, recalling that $K_{\ve} \leq C_{\cY}\,\ve^{-p}$ and observing that, by the way balls are constructed, $N_T$ can never exceed the size of the smallest $\epsilon/2$-cover of $\scX$ (that is, $N_T \leq C_{\cX}\,(\ve/2)^{-d}$), we obtain
\[
\Reg_T(\cF) \leq \sqrt{2\,T\,N_T\,K_{\ve}\,\ln K_{\ve}} + 2\,T\,\ve = \cO\left(\sqrt{T\,\ve^{-(d+p)}\,\ln\frac{1}{\ve} } + T\,\ve\right)\,.
\]
Setting $\ve = \left(\left(\frac{p+d}{2}\right)\frac{\ln T}{T^{1/2}}\right)^{\frac{2}{p+d+2}}$ gives
$
\Reg_T(\cF)
= {\tilde \cO}\left(T^\frac{p+d+1}{p+d+2}\right)
$ 
as claimed.
%
\end{proof}

\section{The Exp3-RTB algorithm for reserve-price optimization}
\label{app:exp3floor}
\cite{CeGeMa-15-ReservePriceOptimization} showed that a regret of order $\tilde{\cO}\big(\sqrt{T}\big)$ can be achieved with high probability for the problem of regret minimization in second-price auctions with i.i.d.~bids, when the feedback received after each auction is the seller's revenue. In this appendix, we show that the same regret rate can be obtained even when the sequence of bids is nonstochastic, provided the feedback also includes the highest bid. We use this result in order to upper bound the contextual regret in Section~\ref{sec:auctions}.

We consider a setting slightly more general than second-price auctions. Fix any unknown sequence $\ell_1,\ell_2,\dots$ of loss functions $\ell_t:[0,1] \to [0,1]$ satisfying the semi-Lipschitz condition,
\begin{equation}
\label{eq:semi-Lipschitz}
	\ell_t(y+\delta) \ge \ell_t(y) - \delta \qquad \text{for all $0 \leq y \leq y+\delta \leq 1$.}
\end{equation}
In each auction instance $t=1,2,\dots$, the learner selects a reserve price $\hat y_t \in \cY = [0,1]$ and suffers the loss $\ell_t(\hat y_t)$. The learner's feedback is $\ell_t(y)$ for all $y \geq \hat y_t$ (i.e., the one-sided full information feedback). As explained in Section~\ref{sec:auctions}, this setting includes online revenue maximization in second-price auctions as a special case when the learner's feedback includes both revenue and highest bid. The learner's regret is defined by
\[
	\Reg_T \eqdef  \E\left[ \sum_{t=1}^T \ell_t(\hat y_t)\right] - \inf_{0\leq y\leq 1} \sum_{t=1}^T \ell_t(y)~,
\] 
where the expectation is with respect to the randomness in the predictions $\hat y_t$.  We introduce the Exp3-RTB algorithm (Algorithm~\ref{alg:exp3floor} below), a variant of Exp3~\citep{AuCBFrSc-02-NonStochasticBandits} exploiting the richer feedback $\big\{\ell_t(y) \,:\, y \geq \hat y_t\big\}$. The algorithm uses a discretization of the action space $[0,1]$ in $K =  \lceil 1/\gamma \rceil$ actions $y_k := (k-1) \gamma$ for $k=1,\dots,K$.

\begin{algorithm2e}[!ht]
\SetKwInOut{Input}{Input}
\SetKwInOut{Init}{Initialization}
\Input{Exploration parameter $\gamma > 0$.}
\Init{Set learning rate $\eta = \gamma/2$ and uniform distribution $p_1$ over $\{1,\dots,K\}$ where $K =  \lceil 1/\gamma \rceil$;}
\For{$t=1,2,\dots$}{
\begin{enumerate}[topsep=0pt,parsep=0pt,itemsep=0pt]
	\item compute distribution $q_t(k) = (1-\gamma) p_t(k) + \gamma \indicator{k=1}$ for $k=1,\dots,K$;
	\item draw $I_t \sim q_t$ and choose $\hat y_t = y_{I_t} = (I_t-1) \gamma$;
	\item for each $k=1,\dots,K$, compute the estimated loss
	\[
		\hat \ell_t(k) = \frac{\ell_t(y_k)}{\sum_{j=1}^k q_t(j)} \indicator{I_t \leq k}
	\]
	\item for each $k=1,\dots,K$, compute the new probability assignment
	\[
		p_{t+1}(k) = \frac{\exp\big(-\eta \sum_{s=1}^t \hat \ell_s(k)\big)}{\sum_{j=1}^K \exp\left(-\eta\sum_{s=1}^t \hat \ell_s(j)\right)}~.
	\]
\end{enumerate}}
\caption{Exp3-RTB (for one-sided full information feedback)}
\label{alg:exp3floor}
\end{algorithm2e}
\begin{theorem} \label{thm:exp3floor}
In the one-sided full information feedback, the Exp3-RTB algorithm tuned with $\gamma > 0$ satisfies
\[
	\Reg_T
\leq
	\gamma T \left(2 +  \frac{1}{4} \ln\frac{e}{\gamma}\right) + \frac{2 \ln \lceil 1/\gamma \rceil}{\gamma}~.
\]
In particular, $\gamma = T^{-1/2}$ gives
$
	\Reg_T = \tilde{\cO}\big(\sqrt{T}\big)
$.
\end{theorem}
\begin{proof}
The proof follows the same lines as the regret analysis of Exp3 in~\citep{AuCBFrSc-02-NonStochasticBandits}. The key change is a tighter control of the variance term allowed by the richer feedback.

Pick any reserve price $y_k = (k-1)\gamma$.
We first control the regret associated with actions drawn from $p_t$ (the regret associated with $q_t$ will be studied as a direct consequence). More precisely, since the estimated losses $\hat{\ell}_t(j)$ are nonnegative, we can apply Lemma~\ref{lem:boundEWA-appe} to get
\begin{equation}
 \sum_{t=1}^T p_{t} \cdot \hat  \ell_t - \sum_{t=1}^T \hat \ell_t(k) \leq \frac{\eta}{2} \sum_{t=1}^T \sum_{j=1}^K p_t(j) \hat \ell_t(j)^2  + \frac{\ln K}{\eta}~.
 \label{eq:boundestimatedloss}
\end{equation}
Writing $\E_{t-1}[\, \cdot \, ]$ for the expectation conditioned on $I_1,\dots,I_{t-1}$, we note that 
\[
	\E_{t-1}\big[\hat \ell_t(j)\big] = \ell_t(y_j)
\quad \text{and} \quad
	\E_{t-1}\big[ p_t(j) \hat{\ell}_t(j)^2\big] = \frac{p_t(j) \ell_t(y_j)^2}{\sum_{i=1}^j q_t(i)} \leq \frac{q_t(j)}{(1-\gamma)\sum_{i=1}^j q_t(i)}~,
\]
where we used the definition of $q_t$ and the fact that $|\ell_t(y_j)| \leq 1$ by assumption. Therefore, taking expectation on both sides of~\eqref{eq:boundestimatedloss} entails, by the tower rule for expectations,
\[
	\E\left[\sum_{t=1}^T p_{t} \cdot \ell_t \right] - \sum_{t=1}^T \ell_t(y_k) \leq \frac{\eta}{2 (1-\gamma)} \sum_{t=1}^T  \E\left[ \sum_{j=1}^K  \frac{q_t(j)}{\sum_{i=1}^j q_t(i)}\right] +  \frac{\ln K}{\eta}~.
\]
Setting $s_{t}(j) \eqdef \sum_{i=1}^{j} q_t(i)$, we can upper bound the sum with an integral,
\begin{align*}
	\sum_{j=1}^K  \frac{q_t(j)}{\sum_{i=1}^j q_t(i)} & = 1 + \sum_{j=2}^K \frac{s_t(j) - s_t(j-1)}{s_t(j)} = 1+ \sum_{j=2}^K \int_{s_t (j-1)}^{s_t(j)} \frac{\dd x}{s_t(j)}  \\
	 & \leq 1 + \sum_{j=2}^K \int_{s_t(j-1)}^{s_t(j)} \frac{\dd x}{x}  = 1 + \int_{q_t(1)}^{1} \frac{\dd x}{x} \leq 1 - \ln q_t(1) \leq 1 + \ln \frac{1}{\gamma}~,
\end{align*}
where we used $q_t(1) \geq \gamma$. Therefore, substituting into the previous bound, we get
\begin{equation}
	\E\left[\sum_{t=1}^T p_{t} \cdot \ell_t \right] - \sum_{t=1}^T \ell_t(y_k) \leq \frac{\eta T \ln(e/\gamma)}{2 (1-\gamma)} + \frac{\ln K}{\eta}~.
\label{eq:Exp3Floor-regretbeforemixing}
\end{equation}
We now control the regret of the predictions $\hat y_t = y_{I_t}$, where $I_t$ is drawn from $q_t = (1-\gamma) p_t  + \gamma \delta_1$. By the tower rule,
\begin{align}
 \E\left[\sum_{t=1}^T  \ell_t(\hat y_t)\right] - \sum_{t=1}^T \ell_t(y_k)
		&  = \E\left[\sum_{t=1}^T \bigl( (1-\gamma) p_t \cdot \ell_t + \gamma \ell_t(y_1) \bigr)\right] - \sum_{t=1}^T \ell_t(y_k) \nonumber \\ 
		& \leq (1-\gamma) \, \E\left[\sum_{t=1}^T  p_t \cdot \ell_t - \sum_{t=1}^T \ell_t(y_k) \right] + \gamma T \nonumber \\
		& \leq \frac{\eta T \ln(e/\gamma)}{2} + \frac{\ln K}{\eta} + \gamma T~,
		\label{eq:regret-yk}
\end{align}
where the last inequality is by~\eqref{eq:Exp3Floor-regretbeforemixing}.

To conclude the proof, we upper bound the regret against any fixed $y \in [0,1]$. Since there exists $k \in \{1,\dots,K\}$ such that $y \in [y_k,y_k+\gamma]$, and since each $\ell_t$ satisfies the semi-Lipschitz condition~\eqref{eq:semi-Lipschitz}, we have $\ell_t(y) \geq \ell_t(y_k) - \gamma$. This gives
\[
	\min_{k=1,\dots,K} \E\left[ \sum_{t=1}^T \ell_t(y_k)\right] \leq \min_{0\leq y\leq 1} \sum_{t=1}^T \ell_t(y)  + \gamma\,T~.
\]
Replacing the last inequality into~\eqref{eq:regret-yk}, and recalling that $K = \lceil 1/\gamma\rceil$ and $\eta = \tfrac{\gamma}{2}$, finally yields
\[
	\Reg_T  \leq \frac{\gamma T}{4}\ln\frac{e}{\gamma} + \frac{2\ln \lceil 1/\gamma \rceil}{\gamma} + 2\gamma T \,.
\]
Choosing $\gamma \approx T^{-1/2}$ concludes the proof.
\end{proof}

\section{Proof of Theorem \ref{t:secondpriceflat}}\label{as:proofsecondpriceflat}
\begin{proof}
The proof is very similar to that of Theorem~\ref{th:flat-lip}. We only highlight the main differences. Let $f \in \cF$ be any 1-Lipschitz function from $\cX$ to $[0,1]$. 
For all $s = 1,\dots,N_T$, we define the constant approximation of $f$ in the $s$-th ball by
$
	\ymin_s = \min_{t \in T_s} f(x_t)
$.
Since $f$ is $1$-Lipschitz and the balls have radius $\epsilon$, we have $\max_{t,t' \in T_s}\big|f(x_t) - f(x_{t'})\big| \le 2\ve$. Hence, for all $t \in T_s$, by the semi-Lipschitz property~\eqref{eq:semi-Lipschitz},
\begin{equation}
	\ell_t\big(f(x_t)\big) \ge \ell_t(\ymin_s) - 2\ve~.
	\label{eq:upper-lipschitz}
\end{equation}
Therefore, 
\begin{align*}
	\sum_{t=1}^T \E[\ell_t(\yhat_t)] - \sum_{t=1}^T \ell_t\big(f(x_t)\big)  
& =
	\sum_{s=1}^{N_T} \sum_{t \in T_s} \Big(\E[\ell_t(\yhat_t)] - \ell_t\big(f(x_t)\big)  \Big)
\\ &=
\sum_{s=1}^{N_T} \underbrace{\sum_{t \in T_s} \Big(\E[\ell_t(\yhat_t)] - \ell_t(\ymin_s)  \Big)}_{R_s}
	+
	\sum_{s=1}^{N_T} \sum_{t \in T_s} \Big(\ell_t(\ymin_s) -  \ell_t\big(f(x_t)\big)\Big)
\\&\le
	\gamma T \left(2 +  \frac{1}{4}\ln\frac{e}{\gamma}\right) + \frac{2N_T \ln \lceil 1/\gamma \rceil}{\gamma} + 2\ve T~.
\end{align*}
Each term $R_s$ is the regret suffered by the $s$-th instance of Exp3-RTB against the constant value $\ymin_s$, and so we bound it using Theorem~\ref{thm:exp3floor} in Appendix~\ref{app:exp3floor} and then sum over $s$ recalling that $\sum_s T_s = T$. The other double sum is bounded by $2T\ve$ using~(\ref{eq:upper-lipschitz}).
We bound $N_T$ as in Theorem~\ref{th:flat-lip} obtaining
\[
	\Reg_T(\cF) \lesssim \gamma T \left(1 + \ln\frac{1}{\gamma}\right) + \frac{\ve^{-d}}{\gamma}\ln\frac{1}{\gamma} + \ve T~.
\]
Finally, choosing $\epsilon = \gamma = T^{-1/(d+2)} \leq 1$ gives
\[
	\Reg_T(\cF) = \tilde{\cO}\left(T^{\frac{d+1}{d+2}}\right)
\]
concluding the proof.
\end{proof}

\section{Exp4 regret bound scaling with the range of the losses}
\label{sec:Exp4range}
In this section we revisit the Exp4 algorithm~\citep[Figure~4.1]{Bubeck2012} and prove a regret bound for the one-sided full information model that scales with the range of the losses. In view of the application to chaining, we formulate this result in a setting similar to \citep{maillard2011adaptive}. Namely, when the action $I_t$ played at time $t$ is not necessarily drawn from the distribution prescribed by Exp4.

Exp4 (see Algorithm~\ref{alg:exp4}) operates in a bandit setting with $K$ actions and $N$ experts. We assume a total order $1 < \cdots < K$ on the action set $\cK = \{1,\dots,K\}$. At each time step $t=1,2,\dots$ the advice $\xi_t(j,\cdot)$ of expert $j$ is a probability distribution over the $K$ actions. The learner combines the expert advice using convex coefficients $q_t \in \Delta(N)$. These coefficients are computed by Hedge based on the expert losses $\tilde{\ell}_t(j) = \xi_t(j,\cdot)\cdot\hat{\ell}_t$ for $j=1,\dots,N$, where $\hat{\ell}_t$ is a vector of suitably defined loss estimates. The distribution prescribed by Exp4 is $p_t = \sum_{j=1}^N q_t(j) \xi_t(j,\cdot) \in \Delta(K)$, but the learner's play at time $t$ is $I_t \sim p^*_t$ for some other distribution $p^*_t \in \Delta(K)$. The feedback at time $t$ is the vector of losses $\ell_t(i)$ for all $i \geq I_t$.

\begin{algorithm2e}[!ht]
\SetKwInOut{Input}{Input}
\SetKwInOut{Init}{Initialization}
\Input{Learning rate sequence $\eta_1 \geq \eta_2 \geq \cdots > 0$.}
\Init{Set $q_1$ to the uniform distribution over $\{1,\dots,N\}$;}
\For{$t=1,2,\dots$}{
\begin{enumerate}[topsep=0pt,parsep=0pt,itemsep=0pt]
	\item get expert advice $\xi_t(1,\cdot),\dots,\xi_t(N,\cdot) \in \Delta(K)$;
	\item compute distribution $p^*_t$ over $\{1,\dots,K\}$;
	\item draw $I_t \sim p^*_t$ and observe $\ell_t(i)$ for all $i \geq I_t$;
	\item compute loss estimates $\hat{\ell}_t(i)$ for $i=1,\dots,K$;
	\item compute expert losses $\tilde{\ell}_t(j) = \xi_t(j,\cdot)\cdot\hat{\ell}_t$ for $j=1,\dots,N$;
	\item compute the new probability assignment
	\[
		q_{t+1}(j) = \frac{\exp\left(-\eta_{t+1} \sum_{s=1}^t \tilde{\ell}_s(j)\right)}{\sum_{k=1}^N \exp\left(-\eta_{t+1} \sum_{s=1}^t \tilde{\ell}_s(k)\right)} \qquad \text{for each $j=1,\dots,N$.}
	\]
\end{enumerate}}
\caption{Exp4 with unspecified sampling distributions $p^*_t$ (for one-sided full info feedback)}
\label{alg:exp4}
\end{algorithm2e}
In this setting, the goal is to minimize the regret with respect to the performance of the best expert,
\[
	\max_{j=1,\dots,N} \E\!\left[\sum_{t=1}^T p_t \cdot \ell_t - \sum_{t=1}^T \xi_t(j,\cdot) \cdot \ell_t \right]
\]
in expectation with respect to the random draw of $I_1,\dots,I_T$, where $\xi_t(j,\cdot) \cdot \ell_t$ is the expected loss of expert $j$ at time $t$. In view of our chaining application, we defined the quantities $\xi_t(j,i)$ as random variables because the expert advice at time $t$ might depend on the past realizations of $I_1,\dots,I_{t-1}$. For all $k \in \cK$ let
\[
	P^*_t(k) = \sum_{i=1}^k p^*_t(i)~. 
\]
Also, let $\cK_t \equiv \theset{k \in \cK}{ (\exists j)\, \xi_t(j,k) > 0}$ (in chaining, $\cK_t$ is typically small compared to $\cK$).
The next result shows that, when Exp4 is applied to arbitrary real-valued losses, there is a way of choosing the learning rate sequence so that the regret scales with a quantity smaller than the largest loss value. More specifically, the regret scales with a known bound $E$ on the size of the effective range of the losses
\[
	\max_{t=1,\dots,T} \max_{k,k' \in \cK_t} \big|\ell_t(k) - \ell_t(k')\big| \leq E~.
\]
The rich feedback structure $(\ell_t(k), k \geq I_t)$ is crucial to get our result, since it enables us to use $\ell_{t}\big(\max \cK_t\big)$ in the definition of the loss estimates $\hat \ell_{t}(k)$ below. Indeed, as explained in Section~\ref{sec:chaining}, $\hat \ell_{t}(k)$ has to be explicitly computed only for those actions $k$ such that $I_t \leq k$ and $\xi_t(j,k)>0$ for some  $j$, i.e., $k \in \cK_t$. In this case, $I_t \leq k \leq \max \cK_t$ so that $\ell_{t}\big(\max \cK_t\big)$ is observed.
\begin{theorem}
\label{thm:Exp4-node}
Let $1 < \cdots < K$ be a total order on the action set $\cK = \{1,\dots,K\}$. Suppose Exp4 (Algorithm~\ref{alg:exp4}) is run with one-sided full information feedback on an arbitrary loss sequence $\ell_1,\ell_2,\dots$ with loss estimates 
\[
	\hat \ell_{t}(k) = \frac{\ell_{t}(k) - \ell_{t}\big(\max \cK_t\big)}{P^*_t(k)} \, \indicator{I_t \leq k} \qquad \text{for $k=1,\dots,K$}
\] 
and adaptive learning rate
\begin{equation}
\label{eq:adarate}
	\eta_t = \min\left\{ \frac{\gamma}{2E}, \sqrt{\frac{2\big(\sqrt{2}-1\big)(\ln N)}{(e-2)\tilde{V}_{t-1}}} \right\} \qquad \text{for $t \geq 2$}
\end{equation}
for some parameter $E>0$ and where $\tilde{V}_{t-1} = \sum_{s=1}^{t-1} \sum_{j=1}^N q_s(j) \bigl(\tilde{\loss}_s(j) - q_s \cdot \tilde{\loss}_s \bigr)^2$ is the cumulative variance of Hedge up to time $t-1$. If $p^*_t(1) \geq \gamma$ and
\[
	\max_{t=1,\dots,T} \max_{k,k' \in \cK_t} \big|\ell_t(k) - \ell_t(k')\big| \leq E
\]
almost surely for all $t \geq 1$, then, for all $T \geq 1$,
\[
	\max_{j=1,\dots,K} \E\!\left[\sum_{t=1}^T p_t \cdot \ell_t - \sum_{t=1}^T \xi_t(j,\cdot) \cdot \ell_t\right]
\leq
	4E\sqrt{\frac{T \ln N}{\gamma}} + \frac{E}{\gamma}(4\ln N + 1) \,.
\]
\end{theorem}
Note that the above regret bound does not depend on the number $K$ of actions, but instead on a lower bound $\gamma$ on the probability of observing the smallest action.
\begin{proof}
From the definition of $\cK_t$ and because $P^*_t(k) \geq p^*_t(1) \geq \gamma$,
\begin{align*}
	\max_{i,j=1,\dots,N} \big|\tilde{\ell}_t(i) - \tilde{\ell}_t(j)\big|
&=
	\max_{i,j=1,\dots,N} \bigl|\xi_t(i,\cdot) \cdot \hat{\ell}_t - \xi_t(j,\cdot) \cdot \hat{\ell}_t\bigr|
\\ &\leq
	\max_{k,k'\in \cK_t} \bigl|\hat{\ell}_t(k) - \hat{\ell}_t(k')\bigr|
\\ &\leq
	\frac{2E}{\min_{k\in\cK}{P^*_t(k)}} \leq \frac{2E}{\gamma} \;.
\end{align*}
Since
$
	\tilde{E} = (2E)/\gamma
$
is an upper bound on the size of the range of the losses $\tilde{\ell}_t(j)$, we can use the bound of \citep[Theorem~5]{Cesa-Bianchi2007}, which applies to Hedge run on arbitrary real-valued losses with the learning rate~(\ref{eq:adarate}). This gives us
\begin{equation}
\label{eq:EWAregret}
	\sum_{t=1}^T q_t\cdot\tilde{\ell}_t
\leq
	\min_{j=1,\dots,N} \sum_{t=1}^T \tilde{\ell}_t(j) + 4 \sqrt{ (\ln N) \tilde{V}_T} + 2\tilde{E}\ln N + \frac{\tilde{E}}{2}~.
\end{equation}
Note that $q_t\cdot\tilde{\ell}_t = p_t\cdot\hat{\ell}_t$. Moreover, $K_t = \max \cK_t$ is measurable with respect to $(I_1,\ldots,I_{t-1})$, implying
$
	\E_{t-1}\big[ \hat \ell_t(k) \big] = \ell_t(k) - \ell_t(K_t)
$. 
Therefore, taking the expectation on both sides of~\eqref{eq:EWAregret} with respect to the random draw of $I_1,\dots,I_T$, and using Jensen together with $\tilde{V}_T \le \sum_{t=1}^T\sum_{j=1}^N q_t(j)\cdot\tilde{\ell}_t(j)^2$, we get
\begin{align}
\nonumber
	\max_{j=1,\dots,N} &\E\!\left[\sum_{t=1}^T p_t \cdot \ell_t - \sum_{t=1}^T \xi_t(j,\cdot) \cdot \ell_{t}\right]
\\ &\leq
\label{eq:Rtexp4}
	4 \sqrt{ (\ln N) \sum_{t=1}^T \E\!\left[ \sum_{j=1}^N q_t(j) \Big(\xi_t(j,\cdot) \cdot \hat{\ell}_t\Big)^2\right]} + \frac{4E}{\gamma}\ln N + \frac{E}{\gamma}~.
\end{align}
The variance term inside the square root can be upper bounded as follows. Using Jensen again,
\begin{align*}
	\sum_{j=1}^N q_t(j) \Big(\xi_t(j,\cdot) \cdot \hat{\ell}_t\Big)^2  
&\leq
	\sum_{j=1}^N q_t(j) \sum_{k=1}^K \xi_t(j,k) \hat{\ell}_t(k)^2
\\&=
	\sum_{j=1}^N q_t(j) \sum_{k=1}^K \xi_t(j,k) \left(\frac{\ell_t(k) - \ell_t(K_t)}{P^*_t(k)} \right)^2 \indicator{I_t \leq k}
\\&\leq
	\frac{E^2}{\gamma} \sum_{j=1}^N q_t(j)\sum_{k=1}^K  \frac{\xi_t(j,k)}{P^*_t(k)} \, \indicator{I_t \leq k}
\end{align*}
where the last inequality is because $\big| \ell_t(k) - \ell_t(K_t) \big| \leq E$ when $\xi_t(j,k)>0$ and because $P^*_t(k) \geq p^*_t(1) \geq \gamma$. Therefore, recalling $P^*_t(k) = \Prob_{t-1}(I_t \leq k)$,
\[
	\E_{t-1}\left[ \sum_{j=1}^N q_t(j) \big( \xi_t(j,k) \cdot \hat \ell_{t} \big)^2 \right]
\leq
	\frac{E^2}{\gamma} \sum_{j=1}^N q_t(j)\sum_{k=1}^K  \frac{\xi_t(j,k)}{P^*_t(k)} \Prob_{t-1}(I_t \leq k)
=
	\frac{E^2}{\gamma}~.
\]
Substituting the last bound in~\eqref{eq:Rtexp4} concludes the proof.
\end{proof}
Next we extend the previous result to penalized loss estimates, which is useful in Section~\ref{sec:efficient-chaining} to control the variance terms all along the covering tree.
\begin{theorem}
\label{thm:Exp4-penalized}
Let $1 < \cdots < K$ be a total order on the action set $\cK = \{1,\dots,K\}$. Let $E,F > 0$. Consider any penalty $\pen_t \in \R^K$ measurable with respect to $(I_1,\ldots,I_{t-1})$ at time $t$. Suppose Exp4 (Algorithm~\ref{alg:exp4}) is run with one-sided full information feedback on an arbitrary loss sequence $\ell_t\in\R^K$, $t \geq 1$, with loss estimates 
\[
	\hat \ell_{t}(k) = \frac{\ell_{t}(k) - \ell_{t}\big(\max \cK_t\big) + E }{P^*_t(k)} \, \indicator{I_t \leq k} + \pen_t(k) + F  \qquad \text{for $k=1,\dots,K$}
\] 
and constant learning rate $\eta > 0$. If we have, for all $t \geq 1$, almost surely,
\[
	\max_{t=1,\dots,T} \max_{k,k' \in \cK_t} \big|\ell_t(k) - \ell_t(k')\big| \leq E \qquad \textrm{and} \qquad \max_{t=1,\dots,T} \max_{k \in \cK_t} \big|\pen_t(k)\big| \leq F~,
\]
then, for all $T \geq 1$,
\begin{align*}
	\max_{j=1,\dots,K} &\E\!\left[\sum_{t=1}^T p_t \cdot \bigl(\ell_t +\pen_t\bigr) - \sum_{t=1}^T \xi_t(j,\cdot) \cdot \bigl(\ell_t +\pen_t\bigr)\right] \\
&\leq \frac{\ln N}{\eta} + 4 \eta T F^2 + 4 \eta E^2  \sum_{t=1}^T \E\!\left[\sum_{k=1}^K  \frac{p_t(k)}{P^*_t(k)}\right] \,.
\end{align*}
\end{theorem}
\begin{proof}
From the definition of $\cK_t$ and because $E$ and $F$ are upper bounds on the losses and penalties associated with actions in $\cK_t$, 
\begin{align*}
	\min_{j=1,\dots,N} \tilde{\ell}_t(j)
& =
	\min_{j=1,\dots,N} \xi_t(j,\cdot) \cdot \hat{\ell}_t
\geq \min_{k \in \cK_t} \hat{\ell}_t(k) \geq 0 \,.
\end{align*}
We can thus use the regret bound for Hedge with weight vectors $q_t$, constant learning rate $\eta$, and nonnegative losses $\tilde{\ell}_t(k)$ (see~Lemma~\ref{lem:boundEWA-appe} in Appendix~\ref{sec:EWA-nonnegative}):
\begin{equation}
\label{eq:EWAregret-bis}
	\sum_{t=1}^T q_t\cdot\tilde{\ell}_t
\leq
	\min_{j=1,\dots,N} \sum_{t=1}^T \tilde{\ell}_t(j) + \frac{\ln N}{\eta} + \frac{\eta}{2}\sum_{t=1}^T \sum_{j=1}^N q_t(j) \tilde{\ell}_t(j)^2 \,.
\end{equation}
Note that $q_t\cdot\tilde{\ell}_t = p_t\cdot\hat{\ell}_t$. Moreover, $K_t = \max \cK_t$ and $\pen_t$ are measurable with respect to $(I_1,\ldots,I_{t-1})$, implying
$
	\E_{t-1}\big[ \hat \ell_t(k) \big] = \ell_t(k) - \ell_t(K_t) + E + \pen_t(k) + F
$. 
Therefore, taking the expectation on both sides of~\eqref{eq:EWAregret-bis} with respect to the random draw of $I_1,\dots,I_T$, we get
\begin{align}
\nonumber
	\max_{j=1,\dots,N} &\E\!\left[\sum_{t=1}^T p_t \cdot \bigl(\ell_t + \pen_t \bigr) - \sum_{t=1}^T \xi_t(j,\cdot) \cdot \bigl(\ell_t + \pen_t \bigr)\right]
\\ &\leq
\label{eq:Rtexp4-bis}
	\frac{\ln N}{\eta} + \frac{\eta}{2}\sum_{t=1}^T \E\!\left[\sum_{j=1}^N q_t(j) \Big(\xi_t(j,\cdot) \cdot \hat{\ell}_t\Big)^2\right] \,.
\end{align}
The variance term can be upper bounded as follows. Using Jensen's inequality,
\begin{align*}
	\sum_{j=1}^N q_t(j) \Big(\xi_t(j,\cdot) \cdot \hat{\ell}_t\Big)^2  
&\leq
	\sum_{j=1}^N q_t(j) \sum_{k=1}^K \xi_t(j,k) \hat{\ell}_t(k)^2
\\&=
	\sum_{k=1}^K \sum_{j=1}^N q_t(j) \xi_t(j,k) \left(\frac{\ell_t(k) - \ell_t(K_t) + E}{P^*_t(k)} \, \indicator{I_t \leq k} + \pen_t(k) + F \right)^2 
\\&\leq
	8 F^2 + 8 E^2 \sum_{k=1}^K  \frac{p_t(k)}{P^*_t(k)^2} \, \indicator{I_t \leq k} \,,
\end{align*}
where the last inequality is because $\big| \ell_t(k) - \ell_t(K_t) \big| \leq E$ and $\big|\pen_t(k)\big| \leq F$ when $\xi_t(j,k)>0$ and because $(a+b+c+d)^2 \leq 4 \bigl(a^2+b^2+c^2+d^2\bigr)$. Therefore, recalling $P^*_t(k) = \Prob_{t-1}(I_t \leq k)$,
\[
	\E_{t-1}\!\left[ \sum_{j=1}^N q_t(j) \big( \xi_t(j,k) \cdot \hat \ell_{t} \big)^2 \right]
\leq
	8 F^2 + 8 E^2 \sum_{k=1}^K  \frac{p_t(k)}{P^*_t(k)} \,.
\]
Substituting the last bound in~\eqref{eq:Rtexp4-bis} concludes the proof.
\end{proof}
Note that when $\pen_t \equiv 0$ and $p_t^*(1) \geq \gamma$ almost surely for all $t \geq 1$, Theorem~\ref{thm:Exp4-penalized} above used with $\eta = (2 E)^{-1} \sqrt{\gamma \ln(N)/T}$ yields a regret bound of $4 E \sqrt{T \ln(N)/ \gamma}$, similarly to Theorem~\ref{thm:Exp4-node}.

We use yet another corollary in Section~\ref{sec:efficient-chaining}. It follows directly from the choice of $\pen_t(k) = - \alpha/P_t^*(k)$ and $F = \alpha/\gamma$ in Theorem~\ref{thm:Exp4-penalized}.
\begin{corollary}
\label{cor:Exp4-penalized}
Let $1 < \cdots < K$ be a total order on the action set $\cK = \{1,\dots,K\}$. Let $E, \alpha, \gamma > 0$ be three parameters. Suppose Exp4 (Algorithm~\ref{alg:exp4}) is run on an arbitrary loss sequence $\ell_t\in\R^K$, $t \geq 1$, with loss estimates 
\[
	\hat \ell_{t}(k) = \frac{\ell_{t}(k) - \ell_{t}\big(\max \cK_t\big) + E}{P^*_t(k)} \, \indicator{I_t \leq k} - \frac{\alpha}{P^*_t(k)} + \frac{\alpha}{\gamma}  \qquad \text{for $k=1,\dots,K$}
\] 
and constant learning rate $\eta > 0$. If we have, for all $t \geq 1$, almost surely, 
\[
\max_{t=1,\dots,T} \max_{k,k' \in \cK_t} \big|\ell_t(k) - \ell_t(k')\big| \leq E \qquad \textrm{and} \qquad p_t^*(1) \geq \gamma \,,
\]
then, for all $T \geq 1$,
\begin{align*}
	& \max_{j=1,\dots,K} \E\!\left[\sum_{t=1}^T \sum_{k=1}^K p_t(k) \left(\ell_t(k) - \frac{\alpha + 4 \eta E^2}{P^*_t(k)}\right) - \sum_{t=1}^T \sum_{k=1}^K \xi_t(j,k) \left(\ell_t(k) - \frac{\alpha}{P^*_t(k)}\right)\right] \\
& \qquad \leq \frac{\ln N}{\eta} + \frac{4 \eta T \alpha^2}{\gamma^2} \,.
\end{align*}
\end{corollary}


\section{Missing Proofs from Section \ref{sec:chaining}}\label{as:missingproofs}
We prove Theorem~\ref{thm:HierarchicalExp4} and Corollary \ref{cor:HierearchicalExp4}. As we said in the main text, the key contribution of chaining is that it allows us to sum up local regret bounds scaling as the range of the local losses. This is possible because of the richer feedback structure induced by the total order on the actions. We first state a lemma indicating that the losses associated with neighboring nodes are indeed close to one another. Recall that $M$ is the depth of $\cT_{\cF}$.
\begin{lemma} \label{lem:treeDistance}
Let $v \in \cT$ be any node at level $m \in \{0,1,\dots,M-1\}$. Then all leaves $w,w' \in \cL_v$ satisfy $\big\|f_w - f_{w'}\big\|_\infty \leq 2^{-m+2}$. Therefore, $\big|\ell_t(y_{i_t(w)}) - \ell_t(y_{i_t(w')})\big| \leq 2^{-m+3}$ for all $t \geq 1$.
\end{lemma}
\begin{proof}
Consider a path $v = v_m \rightarrow v_{m+1} \rightarrow \cdots \rightarrow v_M = w$ joining $v$ to leaf $w$ in the tree. For each $k=m,\dots,M-1$, since $\cF_k$ is a $(2^{-k})$-covering of $\cF$ in sup norm, we have $\|f_{v_k} - f_{v_{k+1}}\|_\infty \leq 2^{-k}$. Therefore,
\begin{equation*}
	\|f_v - f_w\|_\infty = \left\|\sum_{k=m}^{M-1} (f_{v_k} - f_{v_{k+1}}) \right\|_\infty  
		 \leq \sum_{k=m}^{M-1} \|f_{v_k} - f_{v_{k+1}}\|_\infty \leq \sum_{k=m}^{M-1} 2^{-k} \leq 2^{-m+1}\,.
\end{equation*}
Therefore, since $w'\in \cL_v$ as well,
	$
	\|f_w - f_{w'}\|_\infty = \|f_v - f_w\|_\infty  +  \|f_v - f_{w'}\|_\infty \leq 2^{-m+2}
	$,
which proves the first inequality. Now, since $\ell_t$ is $1$-Lipschitz,
\begin{align*}
	\big|\ell_t(y_{i_t(w)}) - \ell_t(y_{i_t(w')})\big| 
&\leq
	\big|y_{i_t(w)},y_{i_t(w')}\big|
\leq
	\big(2^{-M+1} + \|f_w -f_{w'}\|_\infty \big)
\\&\leq
	\big(2^{-M+1} + 2^{-m+2}\big)
\leq
	2^{-m+3}
\end{align*}
where the second inequality uses the definition of $i_t(v)$ in Section~\ref{sec:chaining} and the fact that $\cK$ is a $(2^{-M})$-covering of $[0,1]$. This concludes the proof.
\end{proof}
We are now ready to prove Theorem~\ref{thm:HierarchicalExp4} from the main text.
\begin{proofref}{Theorem~\ref{thm:HierarchicalExp4}}
Each node $v$ at level $m=0,\dots,M-1$ is running an instance of the variant of Exp4 described in Algorithm~\ref{alg:exp4} (Appendix \ref{sec:Exp4range}) over expert set $\cC_v$, where the advice of $w\in\cC_v$ is $\xi_t(w,\cdot) = p_t(w,\cdot)$, and effective action set $\cK_t = \theset{i}{(\exists w \in \cC_v)\; p_t(w,i) > 0} = \cK_t(v)$. Note that, for any $w\in\cC_v$, the distribution $p_t(w,\cdot)$ is a mixture of actions in $\cL_w \subseteq \cL_v$, so that $\cK_t(v)  \subseteq \theset{i_t(w')}{w'\in\cL_v}$. By Lemma~\ref{lem:treeDistance}, the losses of these actions belong to a range of size $E = 2^{-m+3}$.
Since $p^*_t(1) \geq \gamma$ by definition, we are in position to apply Theorem~\ref{thm:Exp4-node} in Appendix~\ref{sec:Exp4range} with $N = |\cC_v| \le N_{m+1}$, and obtain the bound
\[
	\max_{w \in \cC_v} \E\!\left[ \sum_{t=1}^T p_t(v,\cdot) \cdot \ell_t - \sum_{t=1}^T p_t(w,\cdot) \cdot \ell_t \right]
\leq
	2^{-m+5}\sqrt{\frac{T \ln N_{m+1}}{\gamma}} + \frac{2^{-m+3}}{\gamma}(4\ln N_{m+1} + 1)
\]
(for simplicity, we use $\ell_t$ to denote the vector of elements $\ell_t(y_i)$ for $i=1,\dots,K$). Now consider the path $v_0 \to v_{1} \to \dots \to v_M = w^*$ from the root $v_0$ to the leaf $v_M = w^*$ minimizing $\ell_1(y_{i_1(w)})+\cdots+\ell_T(y_{i_T(w)})$ over $w\in\cL$. Recalling that $p_t(w,i) = \indicator{i=i_t(w)}$ for any leaf $w$, we get
\begin{align}
\nonumber
	\E\!\left[\sum_{t=1}^T p_t(v_0,\cdot) \cdot \ell_t\right] &- \min_{w \in \cL} \sum_{t=1}^T \ell_{t}(y_{i_t(w)})
\\&=
\nonumber
	\sum_{m=0}^{M-1}\E\!\left[ \sum_{t=1}^T p_t(v_m,\cdot) \cdot \ell_t - \sum_{t=1}^T p_t(v_{m+1},\cdot) \cdot \ell_t \right]
\\&\leq
\label{eq:regretNode}
	2^5\sum_{m=0}^{M-1}2^{-m}\left(\sqrt{\frac{T \ln N_{m+1}}{\gamma}} + \frac{1}{\gamma}(\ln N_{m+1} + 1)\right)~.
\end{align}
Now, recalling that $I_t \sim p_t^*$ and $p_t^*(i)\ell_t(y_i) = (1-\gamma)p_t(v_0,i)\ell_t(y_i) + \gamma\,\indicator{i = 1} \ell_t(y_1)$ we get
\begin{align*}
	\E\!&\left[\sum_{t=1}^T p^*_t \cdot \ell_t\right]
\leq
	\min_{w \in \cL} \left\{\sum_{t=1}^T \ell_t\bigl(y_{i_t(w)}\bigr) + \gamma \sum_{t=1}^T \Bigl(\ell_t(y_1) - \ell_t\bigl(y_{i_t(w)}\bigr) \Bigr) \right\}
\\ &+
	2^5\sum_{m=0}^{M-1}2^{-m}\left(\sqrt{\frac{T \ln N_{m+1}}{\gamma}} + \frac{1}{\gamma}(\ln N_{m+1} + 1)\right)~.
\end{align*}
Now clearly $\bigl|\ell_t(y_1) - \ell_t\bigl(y_{i_t(w)}\bigr)\bigr| \leq 1$. Moreover, because $\cL$ is a $(2^{-M}$)-covering of $\cF$ and $\cK$ is a $(2^{-M}$)-covering of $[0,1]$, for any $f \in \cF$ there exists $w \in \cL$ such that $\bigl| \ell_t\bigl(y_{i_t(w)}\bigr) - \ell_t\bigl(f(x_t)\bigr) \bigr| \leq \bigl| y_{i_t(w)} - f(x_t)\bigr| \leq 1 \bigl| y_{i_t(w)} - f_w(x_t)+ f_w(x_t) - f(x_t)\bigr| \leq 2^{1-M}$ by definition of $i_t(w)$. Hence,
\begin{align}
\nonumber
	\E\!\left[\sum_{t=1}^T p^*_t \cdot \ell_t\right] &- \inf_{f \in \cF} \sum_{t=1}^T \ell_{t}\bigl(f(x_t)\bigr)
\\ &\leq
\label{eq:regsum}
	(2^{1-M} + \gamma)T
	+ 2^5\sum_{m=0}^{M-1}2^{-m}\left(\sqrt{\frac{T \ln N_{m+1}}{\gamma}}
	+ \frac{1}{\gamma}(\ln N_{m+1} + 1)\right)~.
\end{align}
We now use $N_{m+1} = \cN_\infty(\cF,2^{-(m+1)})$, and follow the standard chaining approach approximating the sums by integrals,
\begin{align}
\nonumber
	\sum_{m=0}^{M-1} 2^{-m} \sqrt{\ln N_{m+1}} 
&=
	4 \sum_{m=0}^{M-1} \left(2^{-(m+1)} - 2^{-(m+2)}\right) \sqrt{\ln \cN_\infty(\cF,2^{-(m+1)})}
\\&\leq
\label{eq:sumint1}
	4 \sum_{m=0}^{M-1} \int_{2^{-(m+2)}}^{2^{-(m+1)}} \sqrt{\ln \cN_{\infty}(\cF,\epsilon)}\dd \epsilon
\leq
	4 \int_{\gamma/2}^{1/2} \sqrt{\ln \cN_\infty(\cF,\epsilon)}\dd \epsilon
\end{align}
where the second inequality is by monotonicity of $\epsilon \mapsto \ln \cN_\infty(\cF,\epsilon)$ and the last inequality follows from $\gamma \leq 2^{-M}$ due to $M = \lfloor \ln_2(1/\gamma) \rfloor$. Similarly,
\begin{equation}
\label{eq:sumint2}
    \sum_{m=0}^{M-1} 2^{-m} \ln N_{m+1}
\leq
	4 \int_{\gamma/2}^{1/2} \ln \cN_{\infty}(\cF,\epsilon) \dd \epsilon~.
\end{equation}
We conclude the proof by substituting~\eqref{eq:sumint1} and~\eqref{eq:sumint2} into~\eqref{eq:regsum}, and using  $2^{-M} \leq 2 \gamma$.
\end{proofref}

\begin{proofref}{Corollary \ref{cor:HierearchicalExp4}}
The metric entropy satisfies $\ln\cN_\infty(\cF,\epsilon) = \cO\big(\epsilon^{-d}\big)$. Therefore, substituting into the regret bound of Theorem~\ref{thm:HierarchicalExp4} and computing the integrals, the regret satisfies
\[
\Reg_T(\cF)
	\leq 5T\gamma + \left\{ 
		\begin{array}{ll}
			\cO\Big(\sqrt{\frac{T}{\gamma}} + \frac{1}{\gamma}\ln\frac{1}{\gamma}\Big) & \text{if $d=1$} \\
			\cO\Big(\sqrt{\frac{T}{\gamma}}\ln\frac{1}{\gamma} + \gamma^{-2}\Big) & \text{if $d=2$} \\
			\cO\big(\sqrt{T}\gamma^{(1-d)/2} + \gamma^{-d}\big) & \text{if $d\geq 3$.}
        \end{array} 
        \right.
 \]
Optimizing $\gamma$ for the different choices of $d$ concludes the proof.
\end{proofref}

%% file: appendix-morelemmas.tex

\section{Algorithm $\textrm{HierExp4}^\star$ and proof of Theorem~\ref{thm:HierExp4star}}
\label{as:proofHierExp4star}

\subsection{Algorithm $\textrm{HierExp4}^\star$}
We first construct the $\textrm{HierExp4}^\star$ algorithm used in Theorem~\ref{thm:HierExp4star}. It is a variant of HierExp4 based on a special hierarchical covering of $\cF$ that we first define below. Recall that $\cF$ is the set of all $1$-Lipschitz functions from $[0,1]^d$ to $[0,1]$.

In the sequel, we use a dyadic discretization of the input space $[0,1]^d$. For each depth $m = 0,1,\ldots$ we define a partition of $[0,1]^d$ with $2^{m d}$ equal cubes of width $2^{-m}$. Each cube $\cX_m(\sigma_{1:m}) \subseteq [0,1]^d$ at depth $m$ is indexed by $\sigma_{1:m} = (\sigma_1,\ldots,\sigma_m) \in \{1,\ldots,2^d\}^m$. Note that the partitions at different depths are nested. We can thus represent them via a tree $\cP$ whose root $\emptyset$ is labeled by $[0,1]^d$ and whose each node $\sigma_{1:m}$ at depth $m$ is labeled with the cube $\cX_m(\sigma_{1:m})$. The children of node $\sigma_{1:m}$ are the nodes at depth $m+1$ of the form $(\sigma_1,\ldots,\sigma_m,\sigma_{m+1})$ with $\sigma_{m+1} \in \{1,\ldots,2^d\}$. They correspond to sub-cubes of $\cX_m(\sigma_{1:m})$.

\paragraph{Wavelet-like approximation of  $\cF$.} Using the above dyadic discretization $\cP$, we approximate any $f \in \cF$ with piecewise-constant functions $f_M:[0,1]^d \to [-1/2,3/2]$ of the form
\begin{equation}
\label{eq:wavelet-decomposition}
f_M(x) = \frac{1}{2} + \sum_{m=1}^M \, \sum_{\sigma_{1:m} \in \{1,\ldots,2^d\}^m} \hspace*{-6mm} 2^{-m} c_m(\sigma_{1:m}) \, \indicator{x \in \cX_m(\sigma_{1:m})} \quad \text{where $c_m(\sigma_{1:m}) \in \bigl\{ - 1, 0, 1 \bigr\}$.}
\end{equation}
As shown in the next lemma, the functions $f_M$ form a $\bigl(2^{-M}\bigr)$-covering of $\cF$ in the sup norm $\|g\|_{\infty} = \sup_{x \in [0,1]^d} \big|g(x)\big|$. We use the following important property: for any given $x \in [0,1]^d$, there exists a  unique $\sigma_{1:M} \in \{1,\ldots,2^d\}^M$ such that $x \in \cX_m(\sigma_{1:m})$ for all $m=1,\ldots,M$, so that
\begin{equation}
\label{eq:wavelet-decomposition-givenx}
f_M(x) = \frac{1}{2} + \sum_{m=1}^M 2^{-m} c_m(\sigma_{1:m}) \,.
\end{equation}
\begin{lemma}
\label{lem:approxStar}
Let $f \in \cF$ and $M \geq 1$. There exist coefficients $c_m(\sigma_{1:m}) \in \bigl\{ - 1, 0, 1 \bigr\}$ such that $f_M$ defined in \eqref{eq:wavelet-decomposition} satisfies $\| f_M - f \|_{\infty} \leq 2^{-M}$.
\end{lemma}
\begin{proof} We denote the center of each cube $\cX_m(\sigma_{1:m})$ by $x_m(\sigma_{1:m})$ and prove by induction on $M \geq 1$ that there exist coefficients $c_m(\sigma_{1:m})$, $m = 0, \ldots,M$, such that
\begin{equation}
\Big| f_M\bigl(x_M(\sigma_{1:M})\bigr) - f\bigl(x_M(\sigma_{1:M})\bigr) \Big| \leq 2^{-(M+1)} \,.
\label{eq:approx-centers}
\end{equation}
This directly yields the conclusion, since every $x \in \cX_M(\sigma_{1:M})$ is $2^{-(M+1)}$-close to the center $x_M(\sigma_{1:M})$ in sup norm and $f$ is $1$-Lipschitz, which by~\eqref{eq:approx-centers} and $f_M(x)=f_M\bigl(x_M(\sigma_{1:M})\bigr)$ implies
\begin{align*}
\big| f_M(x) - f(x) \big|  & \leq \big| f_M\bigl(x_M(\sigma_{1:M})\bigr) - f\bigl(x_M(\sigma_{1:M})\bigr) \big| + \big| f\bigl(x_M(\sigma_{1:M})\bigr) - f(x) \big| \\
& \leq 2^{-(M+1)}+2^{-(M+1)} = 2^{-M} \,.
\end{align*}
We now carry out the induction. For $M=1$ and $\sigma_1 \in \{1,\ldots,2^d\}$, we set
\[
c_1(\sigma_1) \in \argmin_{c \in \{-1,0,1\}} \left| \frac{1}{2} + \frac{c}{2} - f\bigl(x_1(\sigma_1)\bigr) \right| \,,
\]
which corresponds to projecting the value of $f$ at the center $x_1(\sigma_1)$ onto the coarse grid $\{0,1/2,1\}$. Therefore $\big| 1/2 + c_1(\sigma_1)/2 - f(x_1(\sigma_1)) \bigr| \leq 1/4$, which implies~\eqref{eq:approx-centers} by~\eqref{eq:wavelet-decomposition-givenx}.

Now, let $M \geq 2$ and assume that there exist coefficients $c_m(\sigma_{1:m})$,  $m = 0, \ldots,M-1$, such that~\eqref{eq:approx-centers} holds true at level $M-1$. We complete these coefficients at level $M$ as follows: for all $\sigma_{1:{M}} \in \{1,\dots,2^d\}^M$, we set
\[
        c_M(\sigma_{1:M}) \in \argmin_{c \in \{-1,0,1\}} \left|f_{M-1}\big(x_{M}(\sigma_{1:M})\big) + \frac{c}{2^M} - f\bigl(x_{M}(\sigma_{1:M})\bigr) \right| \,.
\]
Note that by~\eqref{eq:wavelet-decomposition-givenx} and $x_M(\sigma_{1:M}) \in \cX_{M}(\sigma_{1:M}) \subset \cX_{M-1}(\sigma_{1:(M-1)})$, we have
\begin{equation}
f_{M}\big(x_{M}(\sigma_{1:M})\big) = f_{M-1}\big(x_{M}(\sigma_{1:M})\big) +  \frac{c_M(\sigma_{1:M})}{2^{M}} \,.
\label{eq:approx-recursion}
\end{equation}
Therefore, the definition of $c_M(\sigma_{1:M})$ above corresponds to making $f_{M}\big(x_{M}(\sigma_{1:M})\big)$ as close to $f\big(x_{M}(\sigma_{1:M})\big)$ as possible. Now, note that $f_{M-1}$ is constant over $\cX_{M-1}(\sigma_{1:(M-1)})$, so that the difference $\Delta_{M-1} \eqdef f_{M-1}\big(x_{M}(\sigma_{1:M})\big)  - f\bigl(x_{M}(\sigma_{1:M})\bigr)$ satisfies
\begin{align*}
     & \big| \Delta_{M-1} \big| = \left|f_{M-1}\big(x_{M-1}(\sigma_{1:(M-1)})\big) - f\bigl(x_{M}(\sigma_{1:M})\bigr) \right| \\
     & \leq \Big|f_{M-1}\big(x_{M-1}(\sigma_{1:(M-1)})\big) - f\bigl(x_{M-1}(\sigma_{1:(M-1)})\bigr) \Big| + \Big| f\bigl(x_{M-1}(\sigma_{1:(M-1)})\bigr) - f\bigl(x_{M}(\sigma_{1:M})\bigr) \Big| \\
     & \stackrel{\eqref{eq:approx-centers}}{\leq} 2^{-M}  + \big\|x_{M-1}(\sigma_{1:(M-1)}) - x_{M}(\sigma_{1:M})\big\|_\infty \\
     & = 2^{-M} + 2^{-(M+1)}\,,
\end{align*}
where the last inequality is by \eqref{eq:approx-centers} at level $M-1$, and where the last equality is by comparison of two cube centers at subsequent depths.

To conclude the proof, we note that the bound $\big|\Delta_{M-1}\big| \leq 2^{-M} + 2^{-(M+1)}$ implies the existence of $c_M(\sigma_{1:M}) \in \{-1,0,1\}$ such that $\big| f_M\bigl(x_M(\sigma_{1:M})\bigr) - f\bigl(x_M(\sigma_{1:M})\bigr) \big| \leq 2^{-(M+1)}$, as required to finish the induction. We can indeed consider three possible cases. If $\big|\Delta_{M-1}\big| \leq 2^{-(M+1)}$, then setting $c_M(\sigma_{1:M}) = 0$ concludes the proof (by~\eqref{eq:approx-recursion}). If $\Delta_{M-1} > 2^{-(M+1)}$, then $\Delta_{M-1}$ lies in the interval $\bigl[2^{-(M+1)}, 2^{-M} + 2^{-(M+1)} \bigr]$, so that setting $c_M(\sigma_{1:M}) = -1$ also concludes the proof (using~\eqref{eq:approx-recursion} again). Similarly we set $c_M(\sigma_{1:M}) = 1$ when $\Delta_{M-1} < - 2^{-(M+1)}$.
\end{proof}
We are now ready to define our hierarchical covering of $\cF$. 

\paragraph{Construction of $\cT_{\cF}^{\star}$.} We build a tree $\cT_{\cF}^{\star}$ such that the value of any function $f_M$ at any point $x \in [0,1]^d$ ---as given by \eqref{eq:wavelet-decomposition-givenx}--- can be computed by following a path from the root to a leaf of $\cT_{\cF}^{\star}$. Unlike Section~\ref{sec:chaining}, this tree is not labeled by functions but instead by either values in $\R$ or cubes $\cX_m(\sigma_{1:m}) \subseteq [0,1]^d$, as illustrated in Figure~\ref{fig:tree2}. More precisely, our tree $\cT_{\cF}^{\star}$ is composed of two types of nodes:
\begin{itemize}
	\item \textit{Nodes at depths $m$}. The root $v_0 = \emptyset$ (at depth $m=0$) is labeled by $1/2$. Each node at depth $m=1,2,\ldots$ is indexed by a tuple $(\sigma_1,c_1,\sigma_2,c_2,\ldots,\sigma_{m},c_m)$ with $\sigma_k \in \{1,\ldots,2^d\}$ and $c_k \in \{-1,0,1\}$; the node is labeled by $1/2+\sum_{k=1}^m 2^{-k} c_k$, which corresponds to \eqref{eq:wavelet-decomposition-givenx}.
	\item \textit{Nodes at depths $m+1/2$}. Each node at depth $m+1/2$, $m \geq 0$, is indexed by a tuple $(\sigma_1,c_1,\sigma_2,c_2,\ldots,\sigma_{m},c_m,\sigma_{m+1})$ with $\sigma_k \in \{1,\ldots,2^d\}$ and $c_k \in \{-1,0,1\}$; the node is labeled by the cube $\cX_m(\sigma_{1:m+1})$.
\end{itemize}
We connect nodes in the following natural way. Nodes $v=(\sigma_1,c_1,\sigma_2,c_2,\ldots,\sigma_{m},c_m)$ at depth~$m$ are connected to all nodes at depth $m+1/2$ of the form $(v,\sigma_{m+1})$, with $\sigma_{m+1} \in \{1,\ldots,2^d\}$. Nodes $w=(\sigma_1,c_1,\sigma_2,c_2,\ldots,\sigma_{m},c_m,\sigma_{m+1})$ at depth $m+1/2$ are connected to all nodes at depth $m+1$ of the form $(w,c_{m+1})$, with $c_{m+1} \in \{-1,0,1\}$.

We stop the tree at depth $M \geq 1$. As previously, we denote by $\cL$ the set of all the leaves of $\cT_{\cF}^{\star}$, by $\cL_v$ the set of all the leaves under $v \in \cT_{\cF}^{\star}$, and by $\cC_v$ the set of children of $v \in \cT_{\cF}^{\star}$.

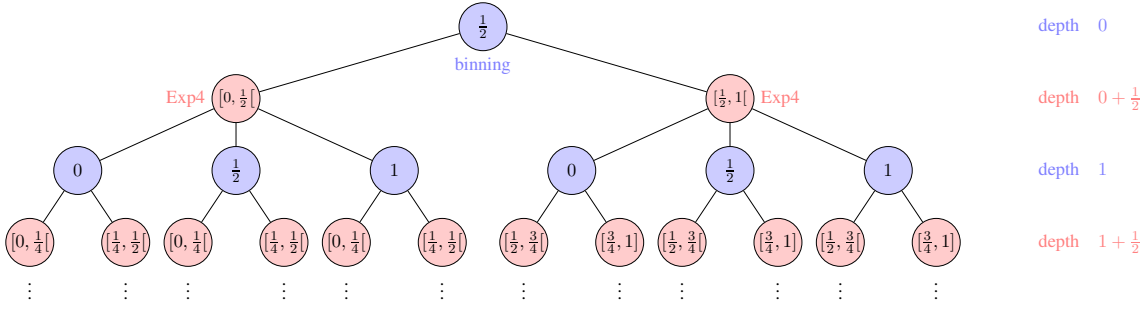
\begin{figure}[!th]
\begin{center}
\resizebox{\textwidth}{!}{
\begin{tikzpicture}[level/.style={sibling distance=60mm/#1,circle,inner sep=0pt, minimum size=1cm},
  level 1/.style={sibling distance=103mm},
  level 2/.style={sibling distance=33mm},
  level 3/.style={sibling distance=20mm},
  level 4/.style={sibling distance=5mm}]

\node [circle,draw,inner sep=0pt, minimum size=1cm,fill=blue!20] (u) {$\frac{1}{2}$}
  child {node [circle,draw,fill=red!20] (a) {\footnotesize $\big[0,\frac{1}{2}\big[$}
    child {node [circle,draw,fill=blue!20] (a1) {$0$}
      child {node [circle,draw,fill=red!20] (a11) {$[0,\frac{1}{4}[$}}
      child {node [circle,draw,fill=red!20] (a12) {$[\frac{1}{4},\frac{1}{2}[$}}
    }
    child {node [circle,draw,fill=blue!20] (a2) {$\frac{1}{2}$}
      child {node [circle,draw,fill=red!20] (a21) {$[0,\frac{1}{4}[$}}
      child {node [circle,draw,fill=red!20] (a22) {$[\frac{1}{4},\frac{1}{2}[$}}
    }
    child {node [circle,draw,fill=blue!20] (a3) {$1$}
      child {node [circle,draw,fill=red!20] (a31) {$[0,\frac{1}{4}[$}}
      child {node [circle,draw,fill=red!20] (a32) {$[\frac{1}{4},\frac{1}{2}[$}}
    }
  }
  child {node [circle,draw,fill=red!20] (b) {\footnotesize $[\frac{1}{2},1[$}
    child {node [circle,draw,fill=blue!20] (b1) {$0$}
      child {node [circle,draw,fill=red!20] (b11) {$[\frac{1}{2},\frac{3}{4}[$}}
      child {node [circle,draw,fill=red!20] (b12) {$[\frac{3}{4},1]$}}
    }
    child {node [circle,draw,fill=blue!20] (b2) {$\frac{1}{2}$}
      child {node [circle,draw,fill=red!20] (b21) {$[\frac{1}{2},\frac{3}{4}[$}}
      child {node [circle,draw,fill=red!20] (b22) {$[\frac{3}{4},1]$}}
   }
    child {node [circle,draw,fill=blue!20] (b3) {$1$}
      child {node [circle,draw,fill=red!20] (b31) {$[\frac{1}{2},\frac{3}{4}[$}}
      child {node [circle,draw,fill=red!20] (b32) {$[\frac{3}{4},1]$}}
   }
};
\draw (u.south) node[below,color=blue!50]{binning};
\draw (b.east) node[right,color=red!50]{Exp4};
\draw (a.west) node[left,color=red!50]{Exp4};

\draw (a11.south) node[below]{$\vdots$};
\draw (a12.south) node[below]{$\vdots$};
\draw (a21.south) node[below]{$\vdots$};
\draw (a22.south) node[below]{$\vdots$};
\draw (a31.south) node[below]{$\vdots$};
\draw (a32.south) node[below]{$\vdots$};
\draw (b11.south) node[below]{$\vdots$};
\draw (b12.south) node[below]{$\vdots$};
\draw (b21.south) node[below]{$\vdots$};
\draw (b22.south) node[below]{$\vdots$};
\draw (b31.south) node[below]{$\vdots$};
\draw (b32.south) node[below]{$\vdots$};


 \begin{scope}[every node/.style={right}]
   \path (u  -| b32) ++(15mm,0) node {} ++(5mm,0) node {$\color{blue!50} \textrm{depth} \quad 0$};
   \path (b  -| b32) ++(15mm,0) node {} ++(5mm,0) node {$\color{red!50} \textrm{depth} \quad 0 + \frac{1}{2}$};
   \path (b3 -| b32) ++(15mm,0) node {} ++(5mm,0) node {$\color{blue!50} \textrm{depth} \quad 1$};
   \path (b32 -| b32) ++(15mm,0) node {} ++(5mm,0) node {$\color{red!50} \textrm{depth} \quad 1 + \frac{1}{2}$};
 \end{scope}

\end{tikzpicture}}
\end{center}
\caption{The structure of the efficient chaining tree in dimension $d=1$. }
\label{fig:tree2}
\end{figure}

\paragraph{Algorithm $\textrm{HierExp4}^{\star}$.} Our new bandit algorithm $\textrm{HierExp4}^{\star}$ (Algorithm~\ref{alg:HierExp4star}) is an efficient modification of HierExp4 based on the tree $\cT_{\cF}^{\star}$ defined above. $\textrm{HierExp4}^{\star}$ alternates between ``binning nodes" at depths~$m$ and ``Exp4 nodes" at depths $m+1/2$ (see Figure~\ref{fig:tree2}).
We discretize the action space $\cY = [0,1]$ with $\cK = \bigl\{k 2^{-M}: k=1,\ldots,2^M\}$ of cardinality $K=2^{M}$. Note that, after projecting onto~$\bigl[2^{-M},1\bigr]$, the set of labels of all the leaves $v \in \cL$ exactly corresponds to this discretization, i.e.,
\begin{align*}
\left\{ \min\!\left( \max\!\left( \frac{1}{2} + \sum_{m=1}^M 2^{-m} c_m, \, \frac{1}{2^M} \right), \, 1 \right): c_m \in \{-1,0,1\} \right\} = \left\{\frac{k}{2^M}: k=1,\ldots,2^{M}\right\} = \cK\,.
\end{align*}
At every round $t$, after observing context $x_t \in [0,1]^d$, the algorithm activates the root $v_0$ and all nodes $\bigl(\sigma_1,c_1,\ldots,\sigma_{m}(,c_m)\bigr)$ such that $x_t \in \cX_m(\sigma_{1:m})$; the other nodes are asleep during that round. In particular, there are $3^M$ activated leaves $v =  (\sigma_1,c_1,\ldots,\sigma_{M},c_M) \in \cL$. Each such leaf recommends the action $i_t(v) \in \{1,\ldots,K\}$ matching its label after projection onto~$\bigl[2^{-M},1\bigr]$, i.e., such that
\begin{equation}
\frac{i_t(v)}{2^M} = \min\!\left( \max\!\left( \frac{1}{2} + \sum_{m=1}^M 2^{-m} c_m, \, \frac{1}{2^M} \right), \, 1 \right) \,.
\label{eq:leaves-actions}
\end{equation}
Each node $v \in \cT_{\cF}^{\star}$ maintains a probability distribution $p_t(v,\cdot) \in \Delta(K)$ over actions in~$\cK$, which is only used and updated when $v$ is activated. The activated leaves $v \in \cL$ correspond to unit masses $p_t(v,i) = \indicator{i=i_t(v)}$. As mentioned earlier, the internal nodes $v \in \cT_{\cF}^{\star} \setminus \cL$ are of two types:
\begin{itemize}
	\item Each internal node $v = (\sigma_1,c_1,\sigma_2,c_2,\ldots,\sigma_{m},c_m)$ at depth $m=0,\ldots,M-1$ is a ``binning node". When activated, it identifies its only child $w = (v,\sigma_{m+1})$ handling $x_t$ (i.e., such that $x_t \in \cX_{m+1}(\sigma_{1:(m+1)})$) and outputs $p_t(v,\cdot) = p_t(w,\cdot)$. We denote by $\cT_{\text{bin}} \subset \cT_{\cF}^{\star} \setminus \cL$ the set of ``binning nodes".
	\item Each internal node $v = (\sigma_1,c_1,\sigma_2,c_2,\ldots,\sigma_{m},c_m,\sigma_{m+1})$ at depth $m+1/2$ is an ``Exp4 node". It runs an instance of Exp4 using the children of $v$ as experts. The variant of Exp4 we use is Algorithm~\ref{alg:exp4} (see Appendix~\ref{sec:Exp4range}) applied to penalized loss estimates (cf.~\eqref{eq:estimatedloss}), which adapts to the range of the losses and allows for a careful control of the variance terms along the tree. Let $\automEWA_v$ be the instance of the Exp4 variant run at node $v$. On all the rounds $t$ when $v$ is activated, this instance updates a distribution $q_t(v,\cdot) \in \Delta(|\cC_v|)$ over experts in $\cC_v$ and outputs $p_t(v,i) = \sum_{w \in \cC_v} q_t(v,w) p_t(w,i)$ for all actions $i = 1,\ldots,K$. We denote by $\cT_{\text{Exp4}} \subset \cT_{\cF}^{\star} \setminus \cL$ the set of ``Exp4 nodes".
\end{itemize}
Note that the root $v_0$ of $\cT_{\cF}^{\star}$ is activated at all rounds $t$. The prediction of $\textrm{HierExp4}^\star$ at time $t$ is $\hat{y}_t = I_t 2^{-M} \in \cK$, where $I_t$ is drawn according to a mixture of $p_t(v_0,\cdot)$ and a unit mass on the minimal action $i=1$. A pseudo-code for $\textrm{HierExp4}^\star$ can be found in Algorithm~\ref{alg:HierExp4star}.

\begin{algorithm2e}[!h]
\SetKwInOut{Input}{Input}
\SetKwInOut{Init}{Initialization}

\Input{Tree $\cT_{\cF}^{\star}$ with root $v_0$ and leaves $\cL$, exploration parameter $\gamma \in (0,1)$, penalization parameters $\alpha_0,\dots,\alpha_{M-1} > 0$, learning rates $\eta_0,\dots, \eta_{M-1} > 0$.
}
\Init{
Set $q_1(v,\cdot)$ to the uniform distribution in $\Delta(|\cC_v|)$ for every $v \in \cT_{\text{Exp4}}$.
} 

\For{$t=1,2,\ldots$}{
\begin{enumerate}[topsep=0pt,parsep=0pt,itemsep=0pt]
  \item Get context $x_t \in \cX$;
  \item Activate $v_0$ and all nodes $v = \big(\sigma_1,c_1,\dots,\sigma_m(,c_m)\big) \in \cT_\cF^{\star}$ such that $x_t \in \cX_m(\sigma_{1:m})$;
  \item Set $p_t(v,i) = \indicator{i=i_t(v)}$ for all $i\in \cK$ and all activated leaves $v \in \cL$;
  \item Set $p_t(v,i) = q_t(v,\cdot) \cdot p_t(\cdot,i)$ for all $i\in \cK$ and all activated nodes $v \in \cT_{\text{Exp4}}$;
  \item Set $p_t(v,\cdot) = p_t(w,\cdot)$ for all activated nodes $v \in \cT_{\text{bin}}$, where $w \in \cC_v$ is the unique activated child of $v$;
  \item Draw $I_t \sim p^*_t$ and play $\hat{y}_t = I_t 2^{-M}$, where $p^*_t(i) = (1-\gamma) p_t(v_0,i) + \gamma\indicator{i = 1}$ for all $i \in \cK$;
  \item Observe $\ell_t(y)$ for all $y \geq \hat{y}_t$;
  \item For all $m = 0,\dots,M-1$ and all activated nodes $v \in \cT_{\text{Exp4}}$ at level $m+1/2$, with $v = (\sigma_1,c_1,\sigma_2,c_2,\ldots,\sigma_{m},c_m,\sigma_{m+1})$, compute the loss estimate for each $i \in \cK_t(v)$,
  \begin{equation}
    \label{eq:estimatedloss1}
    \hat{\ell}_t(v,i) = \frac{\ell_t(i/2^M) - \ell_t\big(j_t(v)/2^{M}\big) + 2^{1-m}}{\sum_{k=1}^i p^*_t(k)}\indicator{I_t \leq i} - \frac{\alpha_m}{\sum_{k=1}^i p^*_t(k)} + \frac{\alpha_m}{\gamma} \,,
  \end{equation}
  where $\cK_t(v) = \theset{i}{(\exists w \in \cC_v)\; p_t(w,i) > 0}$ and $j_t(v) = \max\cK_t(v)$.\\[0.2cm]
  Then, for each $w \in \cC_v$, compute the expert loss
  $
    \tilde{\ell}_t(v,w) = p_t(w,\cdot)\cdot\hat{\ell}_t(v,\cdot)
  $ and perform the update
  \begin{equation}
  \label{eq:EWAweights1}
    q_{t+1}(v,w) = \frac{\displaystyle \exp\!\left(-\eta_m \sum_{s=1}^t \tilde{\loss}_s(v,w) \indicator{x_s \in \cX_{m+1}(\sigma_{1:(m+1)})} \right)}{\displaystyle \sum_{w'\in\cC_v} \exp\!\left(-\eta_m \sum_{s=1}^t \tilde{\loss}_s(v,w') \indicator{x_s \in \cX_{m+1}(\sigma_{1:(m+1)})} \right)} \;.
  \end{equation}
\end{enumerate}
}
\caption{
\label{alg:HierExp4star}
$\textrm{HierExp4}^\star$ (for the one-sided full information feedback)
}
\end{algorithm2e}

\subsection{Proof of Theorem~\ref{thm:HierExp4star}}
As for the proof of Theorem~\ref{thm:HierarchicalExp4}, we start by stating a lemma indicating that the losses associated with neighboring leaves are close to one another.
\begin{lemma}
  \label{lem:treeDistancestar}
  Let $v \in \cT_{\text{Exp4}}$, with depth $m+1/2$, $m \in \{0,\ldots,M-1\}$. Then, all leaves $w,w' \in \cL_v$ in the subtree rooted at $v$ satisfy 
    \[
       \left|\ell_t\left(\frac{i_t(w)}{2^M}\right) - \ell_t\left(\frac{i_t(w')}{2^M}\right) \right| \leq 2^{1-m}
    \]
  at all rounds $t \geq 1$ when $w$ and $w'$ are both activated. 
\end{lemma}
\begin{proof}
Denote $w = (\sigma_1,c_1,\dots,\sigma_M,c_M)$ and $w' = (\sigma_1',c_1',\dots,\sigma_M',c_M')$. First, we remark that since $w$ and $w'$ have the common ancestor $v$ at level $m$, we have
$
    \sigma_i  = \sigma'_i  
$
and $c_i = c'_i$
for all $i \leq m$, as well as $\sigma_{m+1}=\sigma'_{m+1}$ (recall that $v \in \cT_{\text{Exp4}}$). Thus,
\[
  \left| \sum_{j=1}^M 2^{-j} c_j - \sum_{j=1}^M 2^{-j} c_j' \right| \leq \sum_{j=m+1}^M 2^{-j} |c_j - c_j'| \leq \sum_{j=m+1}^{M} 2^{-j+1 } \leq 2^{1-m} \,.
\]
Therefore, by definition~\eqref{eq:leaves-actions} of $i_t(w)$ and $i_t(w')$, and since projecting two real numbers onto $\bigl[2^{-M},1\bigr]$ can only reduce their distance, we get, when $w$ and $w'$ are both activated,
\[
     \left|\frac{i_t(w)}{2^M} - \frac{i_t(w')}{2^M} \right| \leq  \left| \frac{1}{2} + \sum_{j=1}^M 2^{-j} c_j - \frac{1}{2} - \sum_{j=1}^M 2^{-j} c_j' \right| \leq 2^{1-m} \,,
\]
which implies the result since $\ell_t$ is 1-Lipschitz.
\end{proof}
We are now ready to prove Theorem~\ref{thm:HierExp4star} from the main text.
\begin{proofref}{Theorem~\ref{thm:HierExp4star}}
Let $f \in \cF$. By Lemma~\ref{lem:approxStar}, we can fix $f_M:[0,1]^d \to \bigl[-1/2,3/2\bigr]$ of the form~\eqref{eq:wavelet-decomposition}, with coefficients $c_m(\sigma_{1:m}) \in \bigl\{- 1, 0, 1 \bigr\}$, and such that $\| f_M - f \|_{\infty} \leq 2^{-M}$.

In the proof of Theorem~\ref{thm:HierarchicalExp4}, we rewrite the regret with respect to $f_M$ as the sum of $M$ regrets along the path joining the root of $\cT_{\cF}$ to the leaf corresponding to~$f_M$. Next, we proceed similarly except that now $f_M$ corresponds to a $(2^d)$-ary subtree of $\cT_{\cF}^{\star}$ indexed by $\sigma_{1:M} \in \{1,\ldots,2^d\}^M$. The $2^{d M}$ leaves  $\bigl(\sigma_1,c_1(\sigma_1),\sigma_2,c_2(\sigma_{1:2}),\ldots,\sigma_M,c_M(\sigma_{1:M})\bigr)$ are labeled with the values of $f_M$ (before projecting onto $\bigl[2^{-M},1\bigr]$) over the cubes $\cX_M(\sigma_{1:M})$.

We control the regret of $\textrm{HierExp4}^\star$ with respect to $f_M$ by starting from the root $v_0$ of $\cT_{\cF}^{\star}$ and by progressively moving down the tree. We write $\Prob_{t-1}$ and $\E_{t-1}$ for conditioning on $I_1,\ldots,I_{t-1}$, and set $P_t^*(k) \eqdef \sum_{i=1}^k p^*_t(i) = \Prob_{t-1}(I_t \leq k)$. In the sequel, we repeatedly use Corollary~\ref{cor:Exp4-penalized} with parameters $E_m \eqdef 2^{1-m}$, $\alpha_m$ and $\gamma$ to control the regret on each ``Exp4 node". This corollary applies to loss estimates penalized with $\pen_t(k) = - \alpha_m/P_t^*(k)$ for ``Exp4 nodes" at depth $m+1/2$, with $m = 0,\ldots,M-1$.

Since the root $v_0$ is a ``binning node", we have $p_t(v_0,k) = p_t(\sigma_1,k)$ at all rounds $t$ such that $x_t \in \cX_1(\sigma_1)$. Therefore, writing $\ell_t(k)$ instead of $\ell_t\bigl(k/2^M\bigr)$,
\begin{align}
& \E\!\left[\sum_{t=1}^T \sum_{k=1}^K p_t(v_0,k) \left(\ell_t(k) - \frac{\alpha_0 + 4 \eta_0 E_0^2}{P^*_t(k)}\right) \right] \nonumber \\
& = \sum_{\sigma_1 \in \{1,\ldots,2^d\}} \E\!\left[\sum_{t:\, x_t \in \cX_1(\sigma_1)} \sum_{k=1}^K p_t(\sigma_1,k) \left(\ell_t(k) - \frac{\alpha_0 + 4 \eta_0 E_0^2}{P^*_t(k)}\right) \right] \nonumber \\
 &  \leq \sum_{\sigma_1 \in \{1,\ldots,2^d\}} \, \Biggl( \min_{c_1 \in \{-1,0,1\}}  \E\!\left[\sum_{t:\, x_t \in \cX_1(\sigma_1)} \sum_{k=1}^K p_t\bigl((\sigma_1,c_1),k\bigr) \left(\ell_t(k) - \frac{\alpha_0}{P^*_t(k)}\right) \right]  \nonumber \\
& \qquad \qquad \qquad \qquad + \frac{\ln 3}{\eta_0} + \frac{4 \eta_0 T(\sigma_1) \alpha_0^2}{\gamma^2} \Biggr) \label{eq:efficientRegret-root1} \\
 & \leq \sum_{\sigma_1} \, \E\!\left[\sum_{t:\, x_t \in \cX_1(\sigma_1)} \sum_{k=1}^K p_t\bigl([\sigma_1,c_1(\sigma_1)],k\bigr) \left(\ell_t(k) - \frac{\alpha_0}{P^*_t(k)}\right) \right]  + \frac{2^d \ln 3}{\eta_0} + \frac{4 \eta_0 T \alpha_0^2}{\gamma^2} \,, \label{eq:efficientRegret-root2}
\end{align}
where \eqref{eq:efficientRegret-root1} follows from Corollary~\ref{cor:Exp4-penalized} with parameters $E_0 = 2$, $\alpha_0$ and $\gamma$, and where $T(\sigma_1)$ is the number of rounds $t$ such that $x_t \in \cX_1(\sigma_1)$. Indeed, the probability distributions $p_t\bigl((\sigma_1,c_1),\cdot\bigr)$ are mixtures of actions $i_t(w)$ associated with activated leaves $w \in \cL_{\sigma_1}$, whose losses are at most at distance $2$ by Lemma~\ref{lem:treeDistancestar} with $m=0$. The last inequality \eqref{eq:efficientRegret-root2} above is obtained by choosing $c_1=c_1(\sigma_1)$.

So far we have showed how to move from depth $m=0$ to depth $1$. Repeating the same argument at all subsequent depths $m=1,\ldots,M-1$, and noting that $\alpha_{m-1} = \alpha_m + 4 \eta_m 2^{2-2m} = \alpha_{m} + 4 \eta_{m} E_{m}^2$, we get
\begin{align}
& \E\!\left[\sum_{t=1}^T \sum_{k=1}^K p_t(v_0,k) \left(\ell_t(k) - \frac{\alpha_0 + 4 \eta_0 E_0^2}{P^*_t(k)}\right) \right] \nonumber \\
 & \leq \sum_{\sigma_{1:M}} \, \E\!\left[\sum_{t:\, x_t \in \cX_1(\sigma_{1:M})} \sum_{k=1}^K p_t\bigl([\sigma_1,c_1(\sigma_1),\ldots,\sigma_M,c_M(\sigma_{1:M})],k\bigr) \left(\ell_t(k) - \frac{\alpha_{M-1}}{P^*_t(k)}\right) \right]  \nonumber \\
& \qquad + \sum_{m=0}^{M-1} \left(\frac{2^{(m+1)d} \ln 3}{\eta_m} + \frac{4 \eta_m T \alpha_m^2}{\gamma^2} \right) \nonumber \\
 & \leq \sum_{\sigma_{1:M}} \, \left(\sum_{t:\, x_t \in \cX_1(\sigma_{1:M})} \ell_t\bigl(i_t(v)\bigr) \right) + \sum_{m=0}^{M-1} \left(\frac{2^{(m+1)d} \ln 3}{\eta_m} + \frac{4 \eta_m T \alpha_m^2}{\gamma^2} \right) \,, \label{eq:efficientRegret-root3}
\end{align}
where in the last inequality $v$ denotes the leaf $\bigl(\sigma_1,c_1(\sigma_1),\ldots,\sigma_M,c_M(\sigma_{1:M})\bigr)$, whose probability distribution $p_t(v,\cdot)$ is the unit mass on its action $i_t(v)$.

We define $\tilde{f}_M(x) \eqdef \min\bigl(\max\bigl(f_M(x), 2^{-M} \bigr), 1\bigr)$ to be the projection of $f_M(x)$ onto $\bigl[2^{-M},1\bigr]$, and, similarly,  $\tilde{f}(x) \eqdef \min\bigl(\max\bigl(f(x), 2^{-M} \bigr), 1\bigr)$. Since projection can only reduce the distance between two points, $\big|\tilde{f}_M(x) - \tilde{f}(x) \big| \leq \big|f_M(x) - f(x) \big| \leq \| f_M - f \|_{\infty} \leq 2^{-M}$. But, noting that $\tilde{f}(x)$ is $2^{-M}$-close to $f(x) \in [0,1]$, this yields $\big|\tilde{f}_M(x) - f(x) \big| \leq \big|\tilde{f}_M(x) - \tilde{f}(x) \big| + \big|\tilde{f}(x) - f(x) \big| \leq 2^{1-M}$. Therefore, since $\ell_t$ is $1$-Lipschitz, and using both~\eqref{eq:wavelet-decomposition-givenx} and the definition~\eqref{eq:leaves-actions} of $i_t(v)$, we get that, when the leaf $v = \bigl(\sigma_1,c_1(\sigma_1),\ldots,\sigma_M,c_M(\sigma_{1:M})\bigr)$ is activated,
\[
\ell_t\bigl(i_t(v)\bigr) = \ell_t\bigl(\tilde{f}_M(x_t)\bigr) \leq \ell_t\bigl(f(x_t)\bigr) + \big|\tilde{f}_M(x_t) - f(x_t) \big| \leq \ell_t\bigl(f(x_t)\bigr) + 2^{1-M} \,.
\]
We plug the last bound into~\eqref{eq:efficientRegret-root3} and multiply the resulting inequality by $1-\gamma$. Then, using $\E_{t-1}\bigl[\ell_t\bigl(\hat{y}_t\bigr)\bigr] = (1-\gamma) p_t(v_0,\cdot) \cdot \ell_t + \gamma \ell_t(1)$, together with $\ell_t(y) \in [0,1]$, we obtain
\begin{align}
\Reg_T(\cF)
& \leq \sum_{m=0}^{M-1} \left(\frac{2^{(m+1)d} \ln 3}{\eta_m} + \frac{4 \eta_m T \alpha_m^2}{\gamma^2} \right) + (1-\gamma) \, \E\!\left[\sum_{t=1}^T \sum_{k=1}^K p_t(v_0,k) \frac{\alpha_0 + 4 \eta_0 E_0^2}{P^*_t(k)} \right] \nonumber 
\\ &+
	\gamma T + 2^{1-M} T \,. \label{eq:efficientRegret-root4}
\end{align}
Now, we use that by definition $p_t^\ast(k) = (1-\gamma) p_t(v_0,k) + \gamma \indicator{k = 1} \geq (1-\gamma)p_t(v_0,k)$.
Therefore, similarly to the analysis of Exp3-RTB, we can control the variance term 
\begin{multline*}
   (1-\gamma) \sum_{k=1}^K  \frac{p_t(v_0,k)}{P^*_t(k)}  \leq  \sum_{k=1}^K \frac{p_t^\ast(k)}{P_t^\ast(k)} =  1 + \sum_{k=2}^K \frac{P_t^\ast(k) - P_t^\ast(k-1)}{P_t^\ast(k)} 
   = 1 + \sum_{k=2}^K \int_{P_t^\ast(k-1)}^{P_t^\ast(k)} \frac{\dd x}{P_t^\ast(k)} \\
   \leq 1 + \sum_{k=2}^K \int_{P_t^\ast(k-1)}^{P_t^\ast(k)} \frac{\dd x}{x} = 1 +  \int_{P_t^\ast(1)}^{1} \frac{\dd x}{x} \leq 1 - \ln P_t^\ast(1) \leq 1 + \ln \frac{1}{\gamma} = \ln \frac{e}{\gamma} \,,
\end{multline*}
where we used $P_t^\ast(1) = p_t^\ast(1) \geq \gamma$.
Hence, substituting into the previous bound~\eqref{eq:efficientRegret-root4} and recalling that $E_0 = 2$, we get
\begin{equation} \label{eq:efficientRegret-root5}
\Reg_T(\cF) \leq \sum_{m=0}^{M-1} \left(\frac{2^{(m+1)d} \ln 3}{\eta_m} + \frac{4 \eta_m T \alpha_m^2}{\gamma^2} \right) + \,   (\alpha_0 + 16  \eta_0) T \ln \Big(\frac{e}{\gamma} \Big) + \gamma T + 2^{1-M} T \,.
\end{equation}

\input{optimization-parameters.tex}
\end{proofref}

\section{Algorithm HierHedge and proof of Theorem~\ref{thm:HierarchicalEWA}}
\label{as:proofHierarchicalEWA}
\begin{algorithm2e}[!h]
\SetKwInOut{Input}{Input}
\SetKwInOut{Init}{Initialization}

\Input{Tree $\cT_{\cF}$ with root $v_0$ and leaves $\cL$, learning rate sequences $\eta_2(v) \geq \eta_3(v) \geq\cdots > 0$ for $v \in \cT_{\cF} \setminus \cL$.
}
\Init{
Set $q_1(v,\cdot)$ to the uniform distribution in $\Delta(|\cC_v|)$ for every $v \in \cT_{\cF} \setminus \cL$.
} 

\For{$t=1,2,\ldots$}{
\begin{enumerate}[topsep=0pt,parsep=0pt,itemsep=0pt]
	\item Get context $x_t \in \cX$;
	\item Set $p_t(v,i) = \indicator{i=i_t(v)}$ for all $i=1,\dots,K$ and for all $v \in \cL$;
	\item Set $p_t(v,i) = q_t(v,\cdot) \cdot p_t(\cdot,i)$ for all $i=1,\dots,K$ and for all $v \in \cT_{\cF} \setminus \cL$;
	\item Draw $I_t \sim p_t(v_0,\cdot)$;
	\item Observe $\ell_t(i)$ for all $i=1,\dots,K$;
	\item For each $v \in \cT_{\cF} \setminus \cL$ and for each $w \in \cC_v$ compute the expert loss
	$
		\tilde{\ell}_t(v,w) = p_t(w,\cdot)\cdot\ell_t
	$ and perform the update
	\begin{equation}
	\label{eq:Hedgeweights}
		q_{t+1}(v,w) = \frac{\exp\!\left(-\eta_{t+1}(v) \sum_{s=1}^t \tilde{\loss}_s(v,w) \right)}{\sum_{w'\in\cC_v} \exp\!\left(-\eta_{t+1}(v) \sum_{s=1}^t \tilde{\loss}_s(v,w') \right)}
	\end{equation}
\end{enumerate}
}
\caption{
\label{alg:HierarchicalEWA}
HierHedge (for the full information feedback)
}
\end{algorithm2e}
The proof goes along the same lines as the proof of Theorem~\ref{thm:HierarchicalExp4}. Let $v$ be any internal node at depth $m=0,\dots,M-1$. By construction of HierHedge, the distribution $p_t(v,\cdot) \in \Delta(K)$ is computed by Hedge and is supported on the set $\theset{i_t(w)}{w \in \cL_v}$. By Lemma~\ref{lem:treeDistance}, the losses of any pair of actions in this set differ by at most $2^{-m+3}$. Hence, we may apply \cite[Theorem~5]{Cesa-Bianchi2007} with $N = |\cC_v| \leq N_{m+1}$, $E = 2^{-m+3}$, and $\tilde{V}_t \leq E^2$ obtaining
\begin{align*}
    \max_{w \in \cC_v}\E\!\left[\sum_{t=1}^T p_t(v,\cdot) \cdot \ell_t - \sum_{t=1}^T p_t(w,\cdot) \cdot \ell_{t}\right]
\leq
	2^{-m+5}\left(\sqrt{T \ln N_{m+1}} + \ln N_{m+1} + 1\right)~.
\end{align*}
As in the proof of Theorem~\ref{thm:HierarchicalExp4}, we sum the above along a path $v_0 \to v_{1} \to \dots \to v_M = w$ for any leaf $w \in \cL$ and get
\begin{align*}
    \E\!\left[\sum_{t=1}^T p_t(v_0,\cdot) \cdot \ell_t\right] - \min_{w\in\cL}\sum_{t=1}^T \ell_{t}(i_t(w))
\leq
	2^5\sum_{m=0}^{M-1} 2^{-m}\left(\sqrt{T \ln N_{m+1}} + \ln N_{m+1} + 1\right)~.
\end{align*}
Since $\cL$ is a $(2^{-M})$-covering of $\cF$, $2^{-M} \le 2\ve$, and $\cK_{\ve}$ is a $\epsilon$-covering of $\cY$, we get
\begin{align*}
	\E\!\left[\sum_{t=1}^T \ell_t\bigl(y_{I_t}\bigr)\right] - \inf_{f \in \cF} \sum_{t=1}^T \ell_t\bigl(f(x_t)\bigr)
\leq
	5\ve T + 2^5\sum_{m=0}^{M-1} 2^{-m}\left(\sqrt{T \ln N_{m+1}} + \ln N_{m+1} + 1\right)~.
\end{align*}
Overapproximating the sums with integrals similarly to~\eqref{eq:sumint1} and~\eqref{eq:sumint2},
\[
	\Reg_T(\cF)
\leq
	5\ve T + 2^7 \int_{\ve/2}^{1/2} \left( 2\sqrt{T\ln\cN_{\infty}(\cF,x)} + \ln\cN_{\infty}(\cF,x) \right) \dd x~.
\]
If $\cF$ is a set of Lipschitz functions $f:[0,1]^d \to [0,1]^p$, where $[0,1]^d$ is endowed with the norm $\|x-x'\|_\infty$, then $\ln\cN_\infty(\cF,\ve) = \cO\big(p \ve^{-d}\big)$ implying
\[
    \Reg_T(\cF) \leq  5T\epsilon + \left\{ 
        \begin{array}{ll}
            \cO\big(\sqrt{pT} +  p\ln(1/\epsilon)\big) & \text{if $d=1$} \\
            \cO\big(\sqrt{pT}\ln(1/\epsilon) + p\epsilon^{1-d}\big) & \text{if $d=2$} \\
            \cO\big(\sqrt{pT}\epsilon^{1-d/2} + p\epsilon^{1-d}\big) & \text{if $d\geq 3$.}
        \end{array} 
        \right.
\]
Optimizing in $\epsilon$, we obtain the stated result.

%% file: optimization-parameters.tex

\paragraph{Optimization of the parameters.} In the sequel, we show that the prescribed choices of $M$, $\gamma$, $(\alpha_m)$, and $(\eta_m)$ lead to the stated regret bound. Before that, we explain these particular choices by approximately optimizing~\eqref{eq:efficientRegret-root5}, which will lead to~\eqref{eq:defetam} and~\eqref{eq:alpham}. For the sake of readability, we will sometimes write $f \lesssim g$ instead of $f = \cO(g)$. First, recall that $\alpha_{m-1}  = \alpha_m + 4 \eta_m E_m^2  = \alpha_m + 2^{4-2m} \eta_m$. We thus make the approximation $\alpha_m \approx 2^{4-2m} \eta_m$, and start by optimizing in $\eta_m$ the terms inside the sum appearing in~\eqref{eq:efficientRegret-root5}. To do so, we equalize the terms to get the equality
\begin{equation*}
  \frac{2^{(m+1)d} \ln 3}{\eta_m} = \frac{4 \eta_m T \alpha_m^2}{\gamma^2}
\end{equation*}
which we approximate, using the previous remark, by
\begin{equation}
  \label{eq:defetam}
   \frac{2^{m d} }{\eta_m} \approx \frac{T 2^{-4m} \eta_m^3}{\gamma^2}  \quad \textrm{hence our choice}  \quad \eta_m =  c 2^{m(\frac{d}{4}+1)} \gamma^{\frac{1}{2}}T^{-\frac{1}{4}} \,,
\end{equation}
where $c>0$ will be optimized by the analysis.
Now, from these values of $\eta_m$, we can compute $\alpha_m$ by choosing $\alpha_M = 0$ and using the recursion $\alpha_{m-1} = \alpha_m + 2^{4-2m}\eta_m$. This yields
\begin{equation}
  \label{eq:alpham}
  \alpha_m = \sum_{j=m+1}^{M} 2^{4-2j} \eta_j = c 2^4 \gamma^{\frac{1}{2}}T^{-\frac{1}{4}} \sum_{j=m+1}^M 2^{j(\frac{d}{4}-1)} \,.
\end{equation}
We now upper bound~\eqref{eq:efficientRegret-root5} by distinguishing between the two cases $1 \leq d \leq 4$ and $d>4$.
\begin{itemize}
  \item If $d \leq 4$, then the terms inside the sum in~\eqref{eq:alpham} are non-increasing, so that
  \begin{equation}
    \label{eq:boundalphadleq4}
    \alpha_m \leq c  2^4 M_1 \gamma^{\frac{1}{2}}T^{-\frac{1}{4}} 2^{m(\frac{d}{4} - 1)} \,,
  \end{equation}
  where $M_1 = M$ if $2\leq d\leq 4$ and $M_1 = 2 \geq 1/\bigl(2^{1-d/4}-1\bigr)$ if $d=1$.
  The sum in the right-hand side of~\eqref{eq:efficientRegret-root5} then becomes, by substituting the definition of $\eta_m$ (see~\eqref{eq:defetam}) and the above upper bound on $\alpha_m$,
  \begin{align*}
    \sum_{m=0}^{M-1} \left(\frac{2^{(m+1)d} \ln 3}{\eta_m} + \frac{4 \eta_m T \alpha_m^2}{\gamma^2} \right) 
      & \leq \sum_{m=0}^{M-1} \left(\frac{2^{m(\frac{3}{4}d-1)+ d} T^{\frac{1}{4}} \ln 3}{c \gamma^{\frac{1}{2}}} + \frac{2^{10} c^3 M_1^2 2^{m(\frac{3d}{4}-1)} T^{\frac{1}{4}}}{\gamma^{\frac{1}{2}}} \right) \\
      & = \Big(\frac{2^d \ln 3}{c} + 2^{10} M_1^2 c^3\Big) T^{\frac{1}{4}} \gamma^{-\frac{1}{2}} \sum_{m=0}^{M-1} 2^{m(\frac{3}{4}d -1)} \\
      & \stackrel{(\ast)}{\leq} 2^7 M_1^{\frac{1}{2}} T^{\frac{1}{4}} \gamma^{-\frac{1}{2}} \sum_{m=0}^{M-1} 2^{m(\frac{3}{4}d -1)}\\
      & \leq \left\{ 
      \begin{array}{ll}
        2^{11} T^{\frac{1}{4}} \gamma^{-\frac{1}{2}} & \text{if $d=1$} \\
        2^7 M^{\frac{3}{2}} T^{\frac{1}{4}} \gamma^{-\frac{1}{2}} 2^{M(\frac{3}{4}d-1)} & \text{if $2\leq d\leq 4$}
      \end{array} \right.
  \end{align*}
  where inequality ($\ast$) is by using $d \leq 4$ and choosing $c =  2^{-5/4} M_1^{-1/2}$, while the last inequality is because $\sum_{m=0}^\infty 2^{-m/4} \leq 2^3$ for $d=1$. Plugging this inequality into~\eqref{eq:efficientRegret-root5}, upper bounding $\alpha_0$ using~\eqref{eq:boundalphadleq4} and $\eta_0$ using~\eqref{eq:defetam}, and recalling that $M = \lceil \ln_2(1/\gamma) \rceil$, we get
  \begin{itemize}
    \item case $d=1$: choosing $\gamma = T^{-1/2}/\ln(T)$,
  \begin{equation}
    \Reg_T(\cF)
	= \cO\big(\sqrt{T \ln T}\big) \,.
    \label{eq:resultd1}
  \end{equation}
    \item case $2\leq d\leq 4$: choosing $\gamma = T^{-1/(d+2/3)}$,
  \begin{align}
    \Reg_T(\cF) & \lesssim  \ln (1/\gamma)^{\frac{3}{2}} T^{\frac{1}{4}} \gamma^{-\frac{1}{2}} \gamma^{-(\frac{3}{4}d-1)}  + \ln (1/\gamma)^2  \gamma^{\frac{1}{2}} T^{\frac{3}{4}}  + \gamma T \nonumber \\
      &  \lesssim (\ln T)^{\frac{3}{2}} T^{\frac{d-1/3}{d+2/3}}  + (\ln T)^2 T^{\frac{3d/4}{d+2/3}} + T^{\frac{d-1/3}{d+2/3}} \lesssim (\ln T)^{\frac{3}{2}} T^{\frac{d-1/3}{d+2/3}} \,,
      \label{eq:resultd24}
  \end{align}
  where the last inequality is because $3d/4 < d-1/3$ for $d \geq 2$.
  \end{itemize}
  \item If $d \geq  5$, then the terms inside the sum in~\eqref{eq:alpham} are exponentially increasing, so that
  \begin{equation}
    \label{eq:alphamboundgeq4}
    \alpha_m \leq c  2^4 \gamma^{\frac{1}{2}} T^{-\frac{1}{4}} 2^{M(\frac{d}{4}-1)} \frac{2^{\frac{d}{4}-1}}{2^{\frac{d}{4}-1}-1} \leq \frac{c 2^4}{1 - 2^{-1/4}} \gamma^{\frac{1}{2}} T^{-\frac{1}{4}} 2^{M(\frac{d}{4}-1)} \leq c 2^7 \gamma^{\frac{1}{2}} T^{-\frac{1}{4}} 2^{M(\frac{d}{4}-1)}\,.
  \end{equation}
  Then, following the lines of the case $d\leq 4$, by substituting $\eta_m$ (see~\eqref{eq:defetam}) and $\alpha_m$ in the sum in the right-hand side of~\eqref{eq:efficientRegret-root5} 
   \begin{align*}
    \sum_{m=0}^{M-1} &\left(\frac{2^{(m+1)d} \ln 3}{\eta_m} + \frac{4 \eta_m T \alpha_m^2}{\gamma^2} \right)
\\ 
      & \leq \sum_{m=0}^{M-1} \left(\frac{2^{m(\frac{3}{4}d-1)+ d} T^{\frac{1}{4}} \ln 3}{c \gamma^{\frac{1}{2}}} + \frac{2^{16} c^3 2^{m(\frac{d}{4}+1)+M(\frac{d}{2}-2)} T^{\frac{1}{4}}}{\gamma^{\frac{1}{2}}} \right) \\
      & =  \frac{T^{\frac{1}{4}}}{\gamma^{\frac{1}{2}}} \left(\frac{2^d \ln 3}{c} \sum_{m=0}^{M-1} 2^{m(\frac{3}{4}d-1)} + 2^{16}c^3 2^{M(\frac{d}{2}-2)} \sum_{m=0}^{M-1} 2^{m(\frac{d}{4}+1)} \right)\\
      & \leq  \frac{T^{\frac{1}{4}}}{\gamma^{\frac{1}{2}}} \left(\frac{2^{4+d}}{c}  2^{M(\frac{3}{4}d-1)} + 2^{16} c^3 2^{M(\frac{3d}{4}-1)} \right)\\
     & =  \frac{T^{\frac{1}{4}}}{\gamma^{\frac{1}{2}}} \left(\frac{2^{4+d}}{c} +  2^{16} c^3  \right) 2^{M(\frac{3}{4}d-1)}   \\
      & \leq {T^{\frac{1}{4}}}{\gamma^{-\frac{1}{2}}}  2^{M(\frac{3}{4}d-1) + 8 + \frac{3}{4}d} 
  \end{align*}
where in the last inequality we used $c = 2^{d/4-3}$.
Plugging into the regret bound~\eqref{eq:efficientRegret-root5}, upper bounding $\alpha_0$ using~\eqref{eq:alphamboundgeq4} and $\eta_0$ using~\eqref{eq:defetam}, and recalling that $M = \lceil\ln_2(1/\gamma)\rceil$, we get
  \begin{multline}
    \Reg_T(\cF) \lesssim T^{\frac{1}{4}}\gamma^{\frac{1}{2} - \frac{3}{4}d} +  \gamma^{\frac{3}{2}-\frac{d}{4}} T^{\frac{3}{4}} \ln(1/\gamma) + \gamma T \\
     \lesssim T^{\frac{d-1/3}{d+2/3}} + T^{\frac{d-1}{d+2/3}} \ln(T)  + T^{\frac{d-1/3}{d+2/3}} 
     \lesssim  T^{\frac{d-1/3}{d+2/3}} \,,
     \label{eq:resultdgeq5}
  \end{multline}
where the second inequality is by setting $\gamma = T^{-1(d+2/3)}$.
\end{itemize}
Putting together the three cases~\eqref{eq:resultd1},~\eqref{eq:resultd24}, and~\eqref{eq:resultdgeq5} concludes the proof of the regret bound.

\paragraph{Running time of $\textrm{HierExp4}^\star$.}
We only address the case $d \geq 2$.
First note that the total number of ``Exp4 nodes'' in $\cT_{\cF}^\star$ is
\begin{align*}
\big|\cT_{\text{Exp4}}\big| & = \sum_{m=0}^{M-1} \left|\Bigl\{v=(\sigma_1,c_1,\sigma_2,c_2,\ldots,\sigma_{m},c_m,\sigma_{m+1}) \,:\, \sigma_k \in \{1,\ldots,2^d\}, c_k \in \{-1,0,1\} \Bigr\}\right| \\
& = \sum_{m=0}^{M-1} 3^m 2^{d(m+1)} = \frac{1}{3} \frac{\bigl(3 \cdot 2^d\bigr)^{M+1}- 3 \cdot 2^d}{\bigl(3 \cdot 2^d\bigr) - 1} \leq \bigl(3 \cdot 2^d\bigr)^M \,.
\end{align*}
Similarly, $\big|\cT_{\text{bin}}\big| \leq \bigl(3 \cdot 2^d\bigr)^M$ and $\big|\cL\big| \leq \bigl(3 \cdot 2^d\bigr)^M$. Summing the three upper bounds, we get $\big|\cT_{\cF}^\star\big| \leq 3\bigl(3 \cdot 2^d\bigr)^M$. Therefore, using $M = \lceil\ln_2(1/\gamma)\rceil$ and $\gamma = T^{-1/(d+2/3)}$, we obtain
\[
M \leq 1 + \frac{\ln_2 T}{d+2/3} \qquad \textrm{so that} \qquad \big|\cT_{\cF}^\star\big| = \cO\left(T^{\frac{d + \ln_2 3}{d+2/3}}\right)~.
\]
The number of actions in $\cK$ is $2^M = \cO\big(T^{\frac{1}{d+2/3}}\big)$. The running time at every round $t$, which is at most proportional to the total number of tuples $(v,w,i)$ for $(v,i) \in \cT_{\cF}^\star \times \cK$ and $w \in \cC_v$, is at most of the order of $\big|\cT_{\cF}^\star\bigr| \times \bigl|\cK\bigr| \times \max\bigr\{2^d, 3\bigr\}$. Therefore, the running time is at most of the order of $T^{\frac{d + 1 + \ln_2 3}{d+2/3}} \leq T^{1.8}$ for all $d \geq 2$.

A tighter analysis however yields a smaller time complexity than the rate $T^{\frac{d + 1 + \ln_2 3}{d+2/3}}$ derived above. This is because ---from Algorithm~\ref{alg:HierExp4star}--- all elementary computations only need to be computed on \emph{activated} ``Exp4 nodes'' nodes, on their (finite number of) children, and on actions in $\cK$. Fix a round $t \geq 1$. From a computation similar to the one above, but noting that all activated nodes share the same coefficients $\sigma_m$ for all depths $m$, we can see that the total number of activated ``Exp4 nodes'' at round $t$ is at most of
\[
\sum_{m=0}^{M-1} 3^m = \frac{3^M - 1}{2}
\]
so that the running time per round is at most of the order of $3^M 2^M \lesssim T^{\frac{1+\ln_2 3}{d+2/3}}$, which concludes the proof for $d \geq 2$. Similarly, in the case $d=1$, the choice of $\gamma = T^{-1/2}/ (\ln T)$ entails a running time at most of the order of
$
3^M 2^M = \cO\big(\sqrt{T} \ln T)^{1+\ln_2 3}\big) = o\Big(T^{\frac{1+\ln_2 3}{1+2/3}}\Bigr)
$,
so that the stated running time also holds true for $d=1$.